\documentclass[acmsmall, nonacm]{acmart}

\setcopyright{none}

\usepackage{mdframed}
\usepackage{graphbox}
\usepackage{booktabs}
\usepackage{multirow}
\usepackage{algorithm}
\usepackage{algpseudocode}

\usepackage{stackengine}

\usepackage[normalem]{ulem}
\usepackage{array}
\usepackage{graphicx}
\usepackage{textcomp}
\usepackage{ifthen}
\usepackage{xcolor}
\usepackage{blindtext}
\usepackage{listings}
\usepackage{xspace}
\usepackage{balance}

\usepackage{footnote}
\usepackage{comment}
\usepackage{enumitem,kantlipsum}
\makesavenoteenv{tabular}
\makesavenoteenv{table}
\usepackage{booktabs} 
\usepackage{multicol}
\usepackage{multirow}
\usepackage{wrapfig}
\usepackage{hyperref}
\usepackage{svg}
\usepackage{caption}
\usepackage{mathtools}
\usepackage[symbol]{footmisc}
\usepackage{cleveref}
\usepackage{subfig}
\usepackage{microtype}
\usepackage{amsmath,amsfonts}
\usepackage{soul}

\newcommand{\allgather}{\texttt{AllGather}\xspace}

\newcommand{\graytext}[1]{\textcolor{gray}{#1}}

\AtBeginDocument{%
  }

\acmJournal{TOMS}

\DeclareMathOperator{\vect}{vec}
\DeclareMathOperator{\diag}{diag}
\DeclareMathOperator{\blkdiag}{blkdiag}
\DeclareMathOperator{\matdiag}{matdiag}

\newcommand*{\horzbar}{\rule[.5ex]{2.5ex}{0.5pt}}

\begin{document}

\title{A Distributed Data-Parallel PyTorch Implementation of the Distributed Shampoo Optimizer for Training Neural Networks At-Scale}

\author{Hao-Jun Michael Shi}
\authornote{Both authors contributed to the original implementation of this work.}
\email{hjmshi@meta.com}
\affiliation{%
  \institution{Meta Platforms, Inc.}
  \streetaddress{1 Hacker Way}
  \city{Menlo Park}
  \state{California}
  \country{USA}
}

\author{Tsung-Hsien Lee}
\authornotemark[1]
\affiliation{%
  \institution{Independent Researcher}
  \country{USA}
  }
\email{tsung.hsien.lee@gmail.com}

\author{Shintaro Iwasaki}
\authornote{Contributed to the primary distributed performance optimization implementation in this work.}
\email{siwasaki@meta.com}
\affiliation{%
  \institution{Meta Platforms, Inc.}
  \streetaddress{1 Hacker Way}
  \city{Menlo Park}
  \state{California}
  \country{USA}
}

\author{Jose Gallego-Posada}
\authornote{Work was performed while a visiting researcher at Meta Platforms, Inc. Performed experimental ablations.}
\affiliation{
  \institution{Mila \& University of Montreal}
  \streetaddress{6666 Rue Saint-Urbain}
  \city{Montreal}
  \state{Quebec}
  \country{Canada}}
\email{josegp@meta.com}

\author{Zhijing Li}
\email{zhijing@meta.com}
\affiliation{%
  \institution{Meta Platforms, Inc.}
  \streetaddress{1 Hacker Way}
  \city{Menlo Park}
  \state{California}
  \country{USA}
}

\author{Kaushik Rangadurai}
\email{krangadu@meta.com}
\affiliation{%
  \institution{Meta Platforms, Inc.}
  \streetaddress{1 Hacker Way}
  \city{Menlo Park}
  \state{California}
  \country{USA}
}

\author{Dheevatsa Mudigere}
\authornote{Work was performed while at Meta Platforms, Inc.}
\email{dheevatsa@nvidia.com}
\affiliation{%
  \institution{NVIDIA Corporation}
  \streetaddress{2788 San Tomas Expressway}
  \city{Santa Clara}
  \state{California}
  \country{USA}
}

\author{Michael Rabbat}
\email{mikerabbat@meta.com}
\affiliation{%
  \institution{Meta Platforms, Inc.}
  \streetaddress{1 Hacker Way}
  \city{Menlo Park}
  \state{California}
  \country{USA}
}

\renewcommand{\shortauthors}{Shi et al.}

\begin{abstract}
    Shampoo is an online and stochastic optimization algorithm belonging to the AdaGrad family of methods for training neural networks. It constructs a block-diagonal preconditioner where each block consists of a coarse Kronecker product approximation to full-matrix AdaGrad for each parameter of the neural network. In this work, we provide a complete description of the algorithm as well as the performance optimizations that our implementation leverages to train deep networks at-scale in PyTorch. Our implementation enables fast multi-GPU distributed data-parallel training by distributing the memory and computation associated with blocks of each parameter via PyTorch's \texttt{DTensor} data structure and performing an \texttt{AllGather} primitive on the computed search directions at each iteration. This major performance enhancement enables us to achieve at most a 10\% performance reduction in per-step wall-clock time compared against standard diagonal-scaling-based adaptive gradient methods. We validate our implementation by performing an ablation study on training ImageNet ResNet50, demonstrating Shampoo's superiority against standard training recipes with minimal hyperparameter tuning. \\
    
    \noindent Our code is available at \href{https://github.com/facebookresearch/optimizers/tree/main/distributed_shampoo}{\texttt{github.com/facebookresearch/optimizers/tree/main/distributed\_shampoo}}.
\end{abstract}




\maketitle

\section{Introduction}

Adaptive gradient methods (Adam(W), AdaGrad, RMSProp) have been widely adopted as the de-facto methods for training neural networks across a range of applications, including computer vision, natural language processing, and ranking and recommendation \cite{zhang2022opt, naumov2019deep, dosovitskiy2020image}. Originally motivated by the need for per-feature, sparsity-aware learning rates \cite{duchi2011adaptive}, these methods have proven to be especially useful because their hyperparameters are easier to tune with faster convergence in some cases.

The most widely-used versions of adaptive gradient methods involve per-coordinate scaling, which is equivalent to applying a diagonal preconditioner to the stochastic gradient. When training large models typical of deep learning applications, which can have millions to billions of variables, it is tractable to store and apply optimizer states of this order. For example, the optimizers (diagonal) AdaGrad, RMSProp, and Adam(W) all make use of auxiliary states that combined are 2--3 times the size of the model. The auxiliary state tracks either the sum or an exponentially-weighted moving average of functions of each component of the gradient (e.g., the square of the component, or the component's value itself).

On the other hand, it is known that there exists a version of AdaGrad where the preconditioner is a dense full matrix, and this full-matrix version offers stronger theoretical convergence guarantees than diagonal AdaGrad~\cite{duchi2011adaptive}. Its state tracks the sum of the outer product of the stochastic gradient with itself. Consequently, the size of the full-matrix AdaGrad state is quadratic in the model size. Furthermore, the method requires inverting the preconditioner matrix, and so the computional cost is cubic in the model size. Its high memory and computational costs renders full-matrix AdaGrad impractical for deep learning.

The Shampoo algorithm \cite{gupta2018shampoo, anil2021scalable} is an adaptive gradient method for training deep neural networks that fills the gap between diagonal and full-matrix preconditioning by applying two approximations. First, it restricts to block-diagonal preconditioners, where each block preconditions a single layer. Second, Shampoo leverages the special structure of neural network gradients to form a Kronecker product approximation of each preconditioner block, further reducing the memory footprint. Combined, these approximations reduce the cost of Shampoo to approximately 4--7 times the model size, which makes Shampoo feasible for training networks at scale. Whereas diagonal adaptive gradient methods fail to capture any cross-parameter correlations, Shampoo captures some of the correlations within each block. This has led to demonstrable improvements in convergence over previous methods, and has enabled Shampoo's productionization for real-world use-cases, such as in Google's ads recommendations systems \cite{anil2022factory}.

It is worth noting that, although Shampoo involves preconditioning the (stochastic) gradient, the motivation behind Shampoo and full-matrix AdaGrad is distinct from second-order Newton-type methods. Newton-based methods approximate a smooth function locally using Taylor expansions to achieve fast local convergence near a minimizer. On the other hand, adaptive gradient methods like AdaGrad are motivated by the design of preconditioners to maximally decrease the distance to the solution after a fixed number of iterations, specifically for convex non-smooth functions. Furthermore, in machine learning applications, there is a greater emphasis on the initial behavior of the training dynamics, as opposed to other applications of nonlinear programming, which place greater importance on obtaining high-accuracy solutions and fast local convergence~\cite{bottou2018optimization}.

The contribution of this paper is the description and design of a PyTorch implementation of the Distributed Shampoo algorithm. It is designed specifically for distributed data-parallel training using PyTorch's \texttt{DistributedDataParallel} module, where each worker only computes a local subset of gradients (called the \textit{local batch}), and the \textit{global} mini-batch gradient is aggregated across workers. Unlike the JAX implementation, which is optimized for heterogeneous TPU/CPU architectures \cite{anil2021distributed}, the PyTorch implementation is optimized for homogeneous GPU architectures. 

Under standard data parallelism, the cost of the optimizer step is assumed to be marginal relative to the forward and backward passes on the network, and therefore the computation is  replicated across all workers. Indeed, these optimizers' inherent simplicity (implemented through element-wise operations) have enabled highly performant (arguably, ideal) implementations via horizontal and vertical fusion; see NVIDIA's APEX optimizers as an example \cite{nvidia2019apex}. 

Instead, because Shampoo significantly increases the total amount of FLOPs-per-iteration by replacing element-wise operations with matrix operations, Shampoo requires a different set of performance optimizations in order to remain competitive with standard diagonal adaptive gradient methods in terms of wall-clock time. Rather than replicating the optimizer state and computation across all workers, as with standard diagonal scaling methods, our implementation distributes the overall memory and compute of each Shampoo update, only requiring each worker to compute a subset of the search directions (with respect to each parameter) based on a pre-determined greedy assignment, similar to the ZeRO-1 optimizer in \cite{rajbhandari2020zero}. After each worker completes their portion of the work, the search directions for each parameter are collected across all workers; see Figure \ref{fig:all gather}. This enables a performant implementation of Shampoo that is practically applicable for large-scale deep learning training by incurring at most a 10\% increase in wall-clock time per-iteration relative to diagonal adaptive gradient methods.

For machine learning engineers and scientists, this performant implementation offers two potential measurable benefits: (1) faster convergence (in terms of number of iterations and wall-clock time) to a model of the same quality; and/or (2) a nontrivial improvement of the model quality after a fixed number of iterations, with additional training costs but no increase in inference and serving costs. 

\begin{figure}
    \centering
    \includegraphics[width=0.6\textwidth]{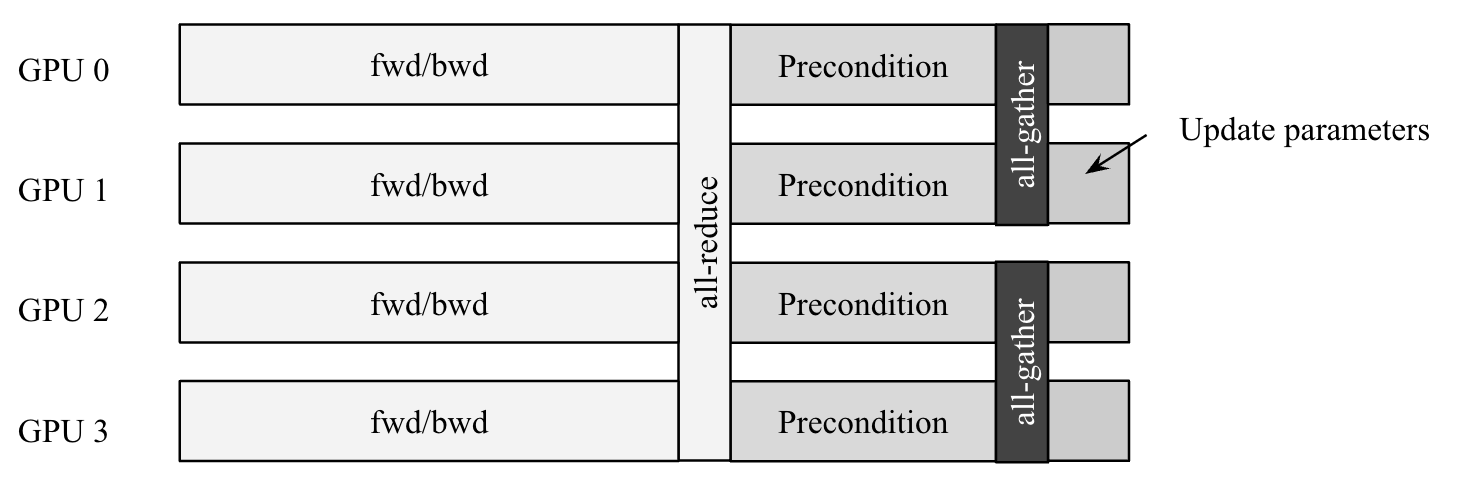}
    \caption{Outline of each distributed data-parallel iteration with the Distributed Shampoo optimizer.}
    \label{fig:all gather}
\end{figure}

\subsection{Main Contributions}

The main contributions of this paper are three-fold:
\begin{enumerate}
    \item We provide a complete characterization of the Distributed Shampoo algorithm, including learning rate grafting as well as important deep learning heuristics (exponential moving averages, momentum, weight decay, etc.) necessary to make Shampoo work well in practice. These are incorporated into our PyTorch Shampoo implementation. Where possible, we provide interpretations of those heuristics; see Sections \ref{sec:overview} and \ref{sec:implementation}.
    \item We describe the main performance optimizations that enable the PyTorch Distributed Shampoo implementation to be competitive with standard diagonal adaptive gradient methods in terms of wall-clock time. This will enable Distributed Shampoo to converge faster than diagonal adaptive gradient methods in terms of wall-clock time (by taking fewer steps than diagonal methods) or achieve better model quality with marginal increases in training time (after the same number of steps); see Section \ref{sec:performance}.
    \item We provide corroborating evidence for Distributed Shampoo's improvement in convergence and model quality by providing ablations and numerical results on ImageNet ResNet50 with standard benchmark training recipes; see Section \ref{sec:numerical}. Specifically, Shampoo over 60 epochs is able to achieve the same validation accuracy as SGD with Nesterov over 90 epochs with minimal hyperparameter tuning. This yields a 1.35x improvement in overall wall-clock time when training.
\end{enumerate}
Our implementation is available online, and the open-source repository includes a \hyperlink{https://github.com/facebookresearch/optimizers/blob/main/distributed_shampoo/README.md}{\texttt{README}} and user guide which complement the discussion in this paper. For details, see: \\
\href{https://github.com/facebookresearch/optimizers/tree/main/distributed_shampoo}{\texttt{https://github.com/facebookresearch/optimizers/tree/main/distributed\_shampoo}}.

\subsection{Terminology and Notation}

For a vectors or matrices $A, B \in \mathbb{R}^{m \times n}$, we define the element-wise square operator $A^{\odot 2} \in \mathbb{R}^{m \times n}$, division operator $A / B \in \mathbb{R}^{m \times n}$, and square-root operator $\sqrt{A} \in \mathbb{R}^{m \times n}$ element-wise, i.e., $[A^{\odot 2}]_{ij} = A_{ij}^2$, $[A / B]_{ij} = A_{ij} / B_{ij}$, and $\left[\sqrt{A} \right]_{ij} = \sqrt{A_{ij}}$. This is in contrast to $A^p$, which denotes the matrix $p$-th power of $A$. We will use square brackets to denote $[n] = \{1, ..., n\}$ for $n \in \mathbb{N}$. We let $I_m$ denote the $m$-dimensional identity matrix, $1_m = (1, 1, ..., 1)^T \in \mathbb{R}^m$ denote the ones vector of length $m$, and $0_{m \times n}$ denote a $m \times n$-dimensional zeros matrix.

We define the $\diag: \mathbb{R}^n \rightarrow \mathbb{R}^{n \times n}$ operator as the function that forms a diagonal matrix with the input vector's entries on the diagonal, i.e., if $a = (a_1, ..., a_n)^T \in \mathbb{R}^n$, then
\begin{equation}
    \diag(a) = \begin{pmatrix} a_{11} & 0 & ... & 0 \\ 0 & a_{22} & ... & 0 \\ \vdots & \vdots & \ddots & \vdots \\ 0 & 0 & ... & a_{nn} \end{pmatrix}
    \in \mathbb{R}^{n \times n}.
\end{equation}
For $n_1, n_2, ..., n_l \in \mathbb{N}$, we define $\blkdiag: \mathbb{R}^{n_1 \times n_1} \times ... \times \mathbb{R}^{n_l \times n_l} \rightarrow \mathbb{R}^{(n_1 + ... + n_l) \times (n_1 + ... + n_l)}$ as the operator that forms a block diagonal matrix from square matrices, i.e., if $A_i \in \mathbb{R}^{n_i \times n_i}$ for $i = 1, ..., l$, then: 
\begin{equation}
    \blkdiag(A_1, ..., A_l) = \begin{pmatrix} 
    A_1 & 0 & ... & 0 \\ 
    0 & A_2 & ... & 0 \\ 
    \vdots & \vdots & \ddots & \vdots \\ 0 & 0 & ... & A_l
    \end{pmatrix}
    \in \mathbb{R}^{(n_1 + ... + n_l) \times (n_1 + ... + n_l)}.
\end{equation}
We define a matrix diagonal operator $\matdiag: \mathbb{R}^{n \times n} \rightarrow \mathbb{R}^{n \times n}$ as the operator that returns a matrix of the same shape but only with its diagonal entries and zero elsewhere, i.e., given $A \in \mathbb{R}^{n \times n}$, then:
\begin{equation}
    \matdiag(A) = A \odot I_n =
    \begin{pmatrix}
    a_{11} & 0 & ... & 0 \\ 
    0 & a_{22} & ... & 0 \\ 
    \vdots & \vdots & \ddots & \vdots \\ 0 & 0 & ... & a_{nn}
    \end{pmatrix}
    \in \mathbb{R}^{n \times n},
\end{equation}
where $\odot$ corresponds to element-wise multiplication of two matrices of the same shape. Vectorization of matrices is performed in a row-wise fashion, i.e., if 
\begin{equation}
    A = \left( \begin{array}{ccc} 
    \horzbar & a_1^T & \horzbar \\ 
    \horzbar & a_2^T & \horzbar \\ 
    & \vdots & \\ 
    \horzbar & a_m^T & \horzbar 
    \end{array} \right)
\end{equation}
then $\vect(A) = (\horzbar a_1^T \horzbar, \horzbar a_2^T \horzbar, ..., \horzbar a_m^T \horzbar)^T$.

For matrices $A \in \mathbb{R}^{m \times n}$ and $B \in \mathbb{R}^{q \times r}$, their \textit{Kronecker product} is defined as
\begin{equation}
A \otimes B = \begin{pmatrix} a_{11} B & a_{12} B & ... & a_{1n} B \\ a_{21} B & a_{22} B & ... & a_{2n} B \\ \vdots & \vdots & \ddots & \vdots \\ a_{m1} B & a_{m2} B & ... & a_{mn} B  \end{pmatrix} \in \mathbb{R}^{mq \times nr}.
\end{equation}
There are a few nice properties of Kronecker products and their relationship with row vectorization that we exploit, namely,
\begin{itemize}
    \item If both $A$ and $B$ are square symmetric positive semi-definite matrices, then $(A \otimes B)^p = A^p \otimes B^p$ for $p \geq 0$. If $A$ and $B$ are symmetric positive definite, then this holds for all $p \in \mathbb{R}$.
    \item If $A$ and $B$ are square matrices and $G \in \mathbb{R}^{m \times q}$ is an $m \times q$ matrix, then $\vect(A G B) = (A \otimes B^T) \vect(G)$.
\end{itemize}
We will call $A$ and $B$ the \textit{Kronecker factor matrices}. 

\section{Problem Statement and Shampoo Algorithm}
\label{sec:overview}

\subsection{Neural Network Training}
The neural network training problem can be posed as a stochastic optimization problem of the form:
\begin{equation}
    \min_{w \in \mathbb{R}^d} \left\{ f(w) = \mathbb{E}_{(x, y) \sim \mathcal{D}}[\ell(m(w; x); y)] \right\}
\end{equation}
where $(x, y) \in \mathbb{R}^{d_0} \times \mathbb{R}^{d_n}$ correspond to a feature vector-label pair, $\mathcal{D}$ corresponds to the underlying data distribution, $m: \mathbb{R}^{d} \times \mathbb{R}^{d_0} \rightarrow \mathbb{R}^{d_n}$ represents a neural network model that takes as input the model parameters $w$ and feature vector $x$ and outputs a prediction in $\mathbb{R}^{d_n}$. The loss function $\ell: \mathbb{R}^{d_n} \times \mathbb{R}^{d_n} \rightarrow \mathbb{R}$ measures how well the model's prediction matches the target label $y$. The model is parameterized by a list of tensors $W^{(1)}, ..., W^{(n)}$ with $\vect(W^{(1)}) \in \mathbb{R}^{d^{(1)}}, ..., \vect(W^{(n)}) \in \mathbb{R}^{d^{(n)}}$. Each tensor $W^{(i)}$ will be called a \textit{parameter}, consistent with PyTorch's terminology for the enumerable representation of the tensor list passed into \texttt{torch.optim.Optimizer}. The full list of tensors will be called the model's \textit{parameters}.

We will denote the concatenated vectorized parameters as $w = (\vect(W^{(1)})^T, ..., \vect(W^{(n)})^T)^T \in \mathbb{R}^d$ with $d = \sum_{i = 1}^n d^{(i)}$. Using this language, we say that our network has \textit{$n$ parameters}, but \textit{$d$ variables} or \textit{weights}. We will refer to the learning rate, momentum parameter, etc. as \textit{hyperparameters} to avoid overloading the term parameter.

A simple example of a neural network model is a multi-layer perceptron consisting of linear layers (ignoring the bias terms) of the form:
\begin{equation}
    m(w; x) = W^{(n)} \sigma( W^{(n - 1)} \sigma( ... \sigma( W^{(1)} x) ... ) ),
\end{equation}
where $W^{(i)} \in \mathbb{R}^{d_i \times d_{i - 1}}$ is a parameter, $w = (\vect(W^{(1)})^T, ..., \vect(W^{(n)})^T)^T \in \mathbb{R}^{d}$ with $d^{(i)} = d_i d_{i - 1}$ and $d = \sum_{i = 1}^n d^{(i)} = \sum_{i = 1}^n d_i d_{i - 1}$ is the vector of all parameters of dimension $d$, and $\sigma$ is a componentwise activation function, i.e., $[\sigma(x)]_j = \sigma(x_j)$. For example, a ReLU activation function is defined as $\sigma(x) = \max(x, 0)$. Consistent with the parameter shapes, we will denote $G^{(i)} = \nabla_{W^{(i)}} \ell(m(w; x), y) \in \mathbb{R}^{d_i \times d_{i - 1}}$ as the (mini-batch) stochastic gradient of parameter $i$ and $g = (\vect(G^{(1)})^T, ..., \vect(G^{(n)})^T)^T \in \mathbb{R}^d$ as the (mini-batch) stochastic gradient vector.\footnote{If we use the mini-batch stochastic gradient, then given a global mini-batch size $B$, we would sample a mini-batch of samples $\{(x_i, y_i)\}_{i = 1}^B$ and the mini-batch stochastic gradient would be defined as $G^{(i)} = \frac{1}{B} \sum_{i = 1}^B \nabla_{W^{(i)}} \ell(m(w; x_i), y_i) \in \mathbb{R}^{d_i \times d_{i - 1}}$.} Here, $d^{(i)}$ corresponds to the number of variables within parameter $i$, $d_i$ corresponds to the dimension of the activation after layer or parameter $i$, and $d$ corresponds to the total number of variables in the optimization problem.

Closely related to the stochastic optimization formulation is the online convex optimization problem. These formulations have been shown to be equivalent under certain scenarios~\cite{cesa2004generalization}. The online optimization problem has relevance to settings where online training on streaming data is used in practice, often to fine-tune models. This problem may be formulated as a game where at round $t$, a player makes a prediction $w_t \in \mathbb{R}^d$, receives a loss evaluated at the predicted point $f_t(w_t)$ (and its gradient $\nabla f_t(w_t)$), and updates their prediction for the next round $w_{t + 1}$. The functions must belong to a predetermined bounded function class $f_t \in \mathcal{F}$, but, unlike in the stochastic optimization setting, are not assumed to arise from some underlying probability distribution. This setting can therefore model settings where the underlying data distribution may shift during training, as in ranking and recommendation 
\cite{naumov2019deep}. 

\subsection{Diagonal Adaptive Gradient Methods}
The family of adaptive gradient methods \cite{duchi2011adaptive, kingma2014adam, dozat2016incorporating, reddi2019convergence} is designed for both the stochastic optimization and online convex optimization. The AdaGrad method preconditions the (sub)gradient by the pseudo-inverse square-root of the sum of squared gradients or gradient outer products, i.e., 
\begin{equation}
\label{eq:adagrad}
    w_{t + 1} = w_t - \alpha_t A_t^{\dagger/2} g_t
\end{equation}
where $g_t \in \mathbb{R}^d$ is the vectorized (mini-batch) stochastic gradient, $\alpha_t > 0$ is the learning rate or steplength, and
\begin{equation}
    A_t = 
    \begin{cases}
    \sum_{s = 0}^t g_s g_s^T & \text{(Full-Matrix AdaGrad)} \\
    \sum_{s = 0}^t \diag(g_s^{\odot 2}) = \sum_{s = 0}^t \matdiag(g_s g_s^T) & \text{(Diagonal AdaGrad)}
    \end{cases}
\end{equation}
for $\epsilon > 0$. In this case, $p_t = A_t^{\dagger/2} g_t$ is the adaptive gradient \textit{search direction}. Note that full-matrix AdaGrad requires $O(d^2)$ memory and $O(d^3)$ computation, while diagonal AdaGrad requires $O(d)$ memory and $O(d)$ computation. Related methods like RMSProp and Adam use exponential moving averages in place of the summation, i.e.,
\begin{equation}
    A_t = \beta_2 A_{t - 1} + (1 - \beta_2) \diag(g_t^{\odot 2}),
\end{equation}
with $A_{-1} = 0$ and may incorporate a bias correction term. Since $d$ is commonly on the order of billions or even trillions of parameters, full-matrix AdaGrad is not practically feasible, and its diagonal approximation is commonly applied instead. In the diagonal case, AdaGrad's optimizer state is instantiated with the same shapes as the neural network parameters $A_t = \diag((\vect(A_t^{(1)})^T, ..., \vect(A_t^{(n)})^T)^T)$ with $\dim(A_t^{(i)}) = \dim(G_t^{(i)})$ for $i = 1, ..., n$, and the algorithm update is implemented in a per-parameter fashion, i.e.,
\begin{equation}
    W_{t + 1}^{(i)} = W_t^{(i)} - \alpha_t G_t^{(i)} / \sqrt{A_t^{(i)}}, ~~~~~~ \forall i = 1, ..., n,
\end{equation}
where division $\cdot / \cdot$ and square-root operators $\sqrt{\cdot}$ are applied componentwise. 

Observe that $A_t$ is symmetric positive semi-definite for both full-matrix and diagonal AdaGrad. Since these methods can only guarantee symmetric positive semi-definiteness, a small regularization term $\epsilon I$ is inserted into the preconditioner to ensure positive-definiteness, either by computing:
\begin{equation}\label{eq:Adagrad}
    w_{t + 1} = w_t - \alpha_t (A_t + \epsilon I)^{-1/2} g_t
\end{equation}
or
\begin{equation}
    w_{t + 1} = w_t - \alpha_t (A_t^{1/2} + \epsilon I)^{-1} g_t.
\end{equation}
Although the latter is more common for (diagonal) AdaGrad, RMSProp, and Adam, we will use the former for Shampoo. 

Since $A_t$ is real symmetric positive semi-definite, the pseudo-inverse square root is defined in terms of its real eigendecomposition 
$A_t = Q_t 
\begin{bmatrix} 
\Lambda_t & 0_{d \times (d - r)} \\
0_{(d - r) \times d} & 0_{(d - r) \times (d - r)}
\end{bmatrix}
Q_t^T$ where $\Lambda_t \in \mathbb{R}^{m \times m}$ for $r \leq d$ is a diagonal matrix consisting of the positive eigenvalues of $A_t$ and $Q_t \in \mathbb{R}^{d \times d}$ is an orthogonal matrix. Note that $r$ is the rank of $A_t$. The matrix pseudo-inverse square root is therefore defined as 
$A_t^{\dagger/2} = Q_t 
\begin{bmatrix} 
\Lambda_t^{-1 / 2} & 0_{d \times (d - r)} \\
0_{(d - r) \times d} & 0_{(d - r) \times (d - r)}
\end{bmatrix}
Q_t^T$
where $\Lambda_t^{-1 / 2}$ is the inverse square root of the diagonal entries in the matrix~\cite{higham2008functions, golub2013matrix}. 

Note that this is \emph{not} equal to the element-wise root inverse, which we denote as $1/\sqrt{\cdot}$. However, when applied to diagonal AdaGrad (with regularization), it is sufficient to take the inverse square root of each diagonal component since $A_t$ is already diagonalized.

\subsection{The Shampoo Algorithm}

Although diagonal AdaGrad is efficient for training, it ignores (uncentered) correlations, and yields a worse constant in its regret bound and convergence rate \cite{duchi2011adaptive}. Full-matrix AdaGrad incorporates these correlations to obtain a better search direction at each step. On the other hand, full-matrix AdaGrad is not tractable due to its quadratic memory and cubic computation requirements. Shampoo provides a scalable solution in between these two regimes by applying two approximations:
\begin{enumerate}
    \item Block-Diagonal Approximation: Rather than constructing a single matrix preconditioner for all parameters simultaneously, Shampoo exploits the parameter list representation of the neural network and constructs a block-diagonal preconditioner where each block preconditions each individual parameter independently. Note that this implies that cross-parameter correlations are ignored.
    \item Kronecker Product Approximation: In order to exploit the underlying tensor structure of each parameter, full-matrix AdaGrad is replaced with a Kronecker product approximation to capture (uncentered) correlations.
\end{enumerate}
For simplicity, let us focus on the multi-layer perceptron case where each parameter consists of a matrix $W^{(i)} \in \mathbb{R}^{d_i \times d_{i - 1}}$ and focus solely on a single parameter $W^{(i)}$. Note that the gradient $G^{(i)} = \nabla_{W^{(i)}} f(w) \in \mathbb{R}^{d_i \times d_{i - 1}}$ shares the same shape as the parameter.

\begin{figure}
    \centering
    \includegraphics[width=0.5\textwidth]{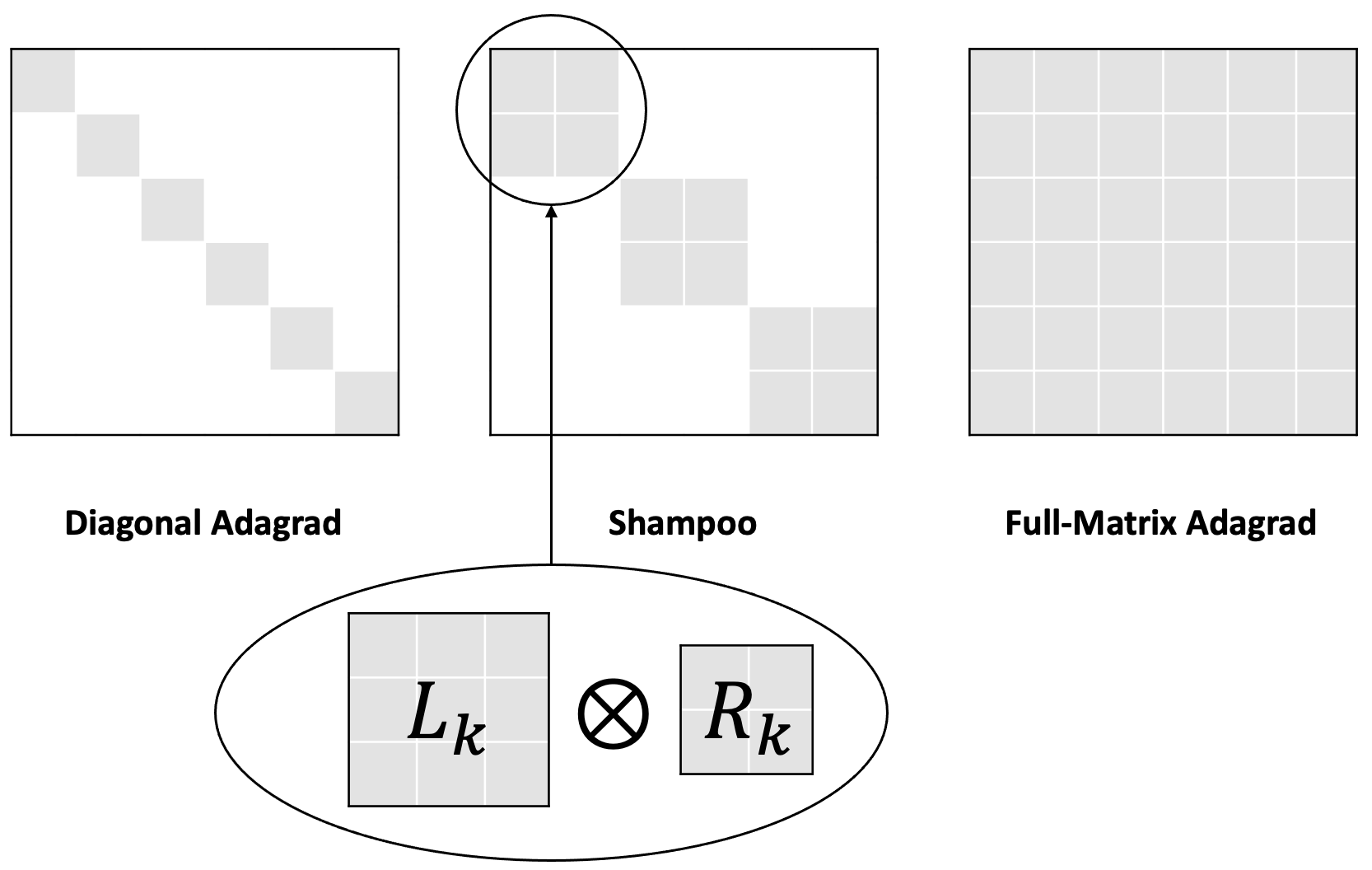}
    \caption{Picture of block-diagonal and Kronecker product approximations used in Shampoo.}
    \label{fig:shampoo approx}
\end{figure}

The gradient of a fully-connected layer for a single data point can be written as the outer product of the pre-activation gradient and the activation before layer $i$. More precisely, we can isolate a single fully-connected layer as the only parameter in the objective function with all other parameters fixed, i.e., $f^{(i)}(W^{(i)}) = \phi^{(i)}(W^{(i)} a^{(i - 1)})$, where $\phi^{(i)}: \mathbb{R}^{d_{i}} \rightarrow \mathbb{R}$ is composed of the loss function and the rest of the model and $a^{(i - 1)}$ is the activation before layer $i$; see Appendix \ref{app: kronecker products} for their precise definition for multi-layer perceptrons. Then the gradient can be written as $G^{(i)} = \nabla_{W^{(i)}} f^{(i)}(W^{(i)}) = \nabla \phi^{(i)}(z)|_{z = W^{(i)} a^{(i - 1)}} (a^{(i - 1)})^T$, and its row vectorization is $g = \vect(G^{i}) = \nabla \phi^{(i)}(z)|_{z = W^{(i)} a^{(i - 1)}} \otimes a^{(i - 1)}$. 

Let the subscript $s$ denote the gradient, function, or activation at iteration $s$. Then full-matrix AdaGrad for layer $i$ accumulates a summation of Kronecker products:
\begin{align*}
    A_t^{(i)} & = \sum_{s = 0}^t g_s g_s^T \\
    & = \sum_{s = 0}^t (\nabla \phi^{(i)}_s(z)|_{z = W^{(i)}_s a^{(i - 1)}_s} \otimes a^{(i - 1)}_s) (\nabla \phi^{(i)}_s(z)|_{z = W^{(i)}_s a^{(i - 1)}_s} \otimes a^{(i - 1)}_s)^T \\
    & = \sum_{s = 0}^t (\nabla \phi^{(i)}_s(z)|_{z = W^{(i)}_s a^{(i - 1)}_s} (\nabla \phi^{(i)}_s(z)|_{z = W^{(i)}_s a^{(i - 1)}_s})^T) \otimes (a^{(i - 1)}_s (a^{(i - 1)}_s)^T).
\end{align*}
We aim to approximate $A_t^{(i)}$ by a single Kronecker product of two \textit{factor matrices} $L_t^{(i)} \in \mathbb{R}^{d_i \times d_i}$ and $R_t^{(i)} \in \mathbb{R}^{d_{i - 1} \times d_{i - 1}}$ such that $A_t^{(i)} \approx L_t^{(i)} \otimes R_t^{(i)}$. Rather than vectorizing the gradient, these matrices will operate directly on the tensor (matrix in the fully-connected case) $G_t^{(i)}$. More specifically, $L_t^{(i)}, R_t^{(i)}$ are defined as:
\begin{align}
    L_t^{(i)} & = \sum_{s = 0}^t G_s^{(i)} [G_s^{(i)}]^T + \epsilon I_{d_i}, \\
    R_t^{(i)} & = \sum_{s = 0}^t [G_s^{(i)}]^T G_s^{(i)} + \epsilon I_{d_{i - 1}},
\end{align}
and its Kronecker product is defined as
\begin{equation} \label{eq:Shampoo preconditioner}
    \bar{A}_t^{(i)} = [L_t^{(i)}]^{1/2} \otimes [R_t^{(i)}]^{1/2}
\end{equation}
for all $i = 1, ..., n$. Since both $L_t^{(i)}$ and $R_t^{(i)}$ are symmetric by definition, the transpose can be ignored. Therefore, using the fact that $\vect(L G R^T) = (L \otimes R) \vect(G)$ for arbitrary matrices $L, G, R$ of appropriate shape with equations \eqref{eq:Adagrad} and \eqref{eq:Shampoo preconditioner}, the Shampoo update can be written as:
\begin{equation} \label{eq:shampoo update}
    W_{t + 1}^{(i)} = W_t^{(i)} - \alpha_t [L_t^{(i)}]^{-1/4} G_t^{(i)} [R_t^{(i)}]^{-1/4} ~~~~~~~ \text{for} ~~~ i = 1, ..., n.
\end{equation}
Notice that full-matrix AdaGrad for parameter $W^{(i)}$ costs $O(d_i^2 d_{i - 1}^2)$ memory and $O(d_i^3 d_{i - 1}^3)$ FLOPs-per-iteration. By utilizing this approximation, the memory footprint can be reduced to $O(d_i^2 + d_{i - 1}^2)$ and the amount of computation to $O(d_i^3 + d_{i - 1}^3)$ FLOPs-per-iteration. 

If the update is expanded across all vectorized parameter weights $w$, the full Shampoo update can be rewritten as:
\begin{equation}
    w_{t + 1} = w_t - \alpha_t \bar{A}_t^{-1/2} g_t
\end{equation}
where $\bar{A}_t$ is a block diagonal matrix of the form
\begin{equation}
\begin{aligned}
    \bar{A}_t & = \blkdiag(\bar{A}_t^{(1)}, ..., \bar{A}_t^{(n)}) \\
    & = \blkdiag([L_t^{(1)}]^{1/2} \otimes [R_t^{(1)}]^{1/2}, ..., [L_t^{(n)}]^{1/2} \otimes [R_t^{(n)}]^{1/2}) \\
    & = 
    \begin{bmatrix} [L_t^{(1)}]^{1/2} \otimes [R_{t}^{(1)}]^{1/2} & 0 & ... & 0 \\ 0 & [L_{t}^{(2)}]^{1/2} \otimes [R_{t}^{(2)}]^{1/2} & ... & 0 \\ 0 & 0 & \ddots & 0 \\ 0 & 0 & ... & [L_{t}^{(n)}]^{1/2} \otimes [R_{t}^{(n)}]^{1/2} \end{bmatrix}.
\end{aligned}
\end{equation}
Shampoo generalizes these ideas to models containing parameters of arbitrary tensor order and dimension; see Section 4 in \cite{gupta2018shampoo}.

\subsection{Layer-wise Learning Rate Grafting}

One major requirement to make Shampoo work in practice is the inclusion of \textit{layer-wise learning rate grafting}. Learning rate grafting was introduced in \cite{agarwal2020disentangling} in order to transfer a pre-existing learning rate schedule from a previous method. The idea is to maintain the search direction from another method (called the \textit{grafted method}) and re-scale each layer's Shampoo search direction to the norm of the search direction of the grafted method. Preconditioners for both Shampoo and the grafting method are updated based on the same sequence of iterates. 

From a global perspective, grafting can be re-interpreted as a heuristic block re-scaling. Let $P_{t, \text{Shampoo}}^{(i)}$ denote the Shampoo search direction for block $i$ and iteration $t$.\footnote{If we are operating on a matrix, then $P_{t, \text{Shampoo}}^{(i)} := [L_t^{(i)}]^{-1/4} G_t^{(i)} [R_t^{(i)}]^{-1/4}$, as seen in Section \ref{sec:overview}.} Given a separate grafted method $P_{t, \text{graft}}^{(i)}$, learning rate grafting modifies the Shampoo step to:
\begin{equation}
    W_{t + 1}^{(i)} = W_t^{(i)} - \alpha_t \left\|P_{t, \text{graft}}^{(i)} \right\|_F \frac{P_{t, \text{Shampoo}}^{(i)}}{\left\|P_{t, \text{Shampoo}}^{(i)} \right\|_F}, ~~~~~ \forall i = 1, ..., n.
\end{equation}
Note that $P_{t, \text{graft}}^{(i)}$ is defined based on the iterate sequence from Shampoo $w_t$, not a separate sequence of iterates. We can therefore re-write the full update as
\begin{equation}
    w_{t + 1} = w_t - \alpha_t D_t \bar{A}_t^{-1/2} g_t
\end{equation}
where
\begin{align}
    D_t & = \blkdiag \left(\frac{\|P_{t, \text{graft}}^{(1)}\|_F}{\|P_{t, \text{Shampoo}}^{(1)}\|_F} I_{d^{(1)}}, ..., \frac{\|P_{t, \text{graft}}^{(n)}\|_F}{\|P_{t, \text{Shampoo}}^{(n)}\|_F} I_{d^{(n)}} \right)\\
    & =
    \begin{bmatrix} 
    \frac{\|P_{t, \text{graft}}^{(1)}\|_F}{\|P_{t, \text{Shampoo}}^{(1)}\|_F} I_{d^{(1)}} & 0 & ... & 0 \\ 0 & \frac{\|P_{t, \text{graft}}^{(2)}\|_F}{\|P_{t, \text{Shampoo}}^{(2)}\|_F} I_{d^{(2)}} & ... & 0 \\ \vdots & \vdots & \ddots & \vdots \\ 0 & 0 & ... & \frac{\|P_{t, \text{graft}}^{(n)}\|_F}{\|P_{t, \text{Shampoo}}^{(n)}\|_F} I_{d^{(n)}}
    \end{bmatrix}.
\end{align}

Our implementation supports all standard diagonal-scaling-based optimizers in PyTorch, including AdaGrad, RMSProp, Adam(W), and SGD. In the case of AdaGrad, RMSProp, and Adam, we implement layerwise learning rate grafting by maintaining the diagonal preconditioner for our grafted method $\tilde{A}_t = \blkdiag(\diag(\vect(\tilde{A}_t^{(1)})), ..., \diag(\vect(\tilde{A}_t^{(n)})))$ where $\dim(\tilde{A}_t^{(i)}) = \dim(G_t^{(i)})$ for $i = 1, ..., n$. For example, if we are grafting from AdaGrad, the grafted preconditioner is defined as $\tilde{A}_t^{(i)} = \sum_{s = 0}^t [G_s^{(i)}]^{\odot 2} = \sum_{s = 0}^t [\nabla_{W^{(i)}} f_t(w_s)]^{\odot 2}$. Note that the preconditioner is updated using the stochastic gradients \textit{evaluated at the same sequence of iterates generated and used by Shampoo}; we use $\tilde{A}$ to distinguish this key difference between standard diagonal AdaGrad and the grafted method. In the case of AdaGrad, RMSProp, and Adam grafting, the grafted search direction is defined as $P_{t, \text{graft}}^{(i)} = G_t^{(i)} / ([\tilde{A}_t^{(i)}]^{1/2} + \epsilon 1_{d_i} 1_{d_{i - 1}}^T)$ for parameter $i$, where $\tilde{A}_t^{(i)}$ is the AdaGrad, RMSProp, or Adam second-moment estimate. 

This heuristic makes Shampoo significantly easier to tune given a pre-existing learning rate scheduler. By grafting from the previous optimizer's learning rate schedule, one generally sees immediate improvements in convergence with little additional hyperparameter tuning. This can be used as an easy baseline for further fine-tuning of the optimizer. For more details, please refer to \cite{agarwal2020disentangling, anil2021scalable}.

A high-level pseudocode for the Shampoo algorithm (with standard accumulation of the factor matrices and AdaGrad learning rate grafting) is provided in Algorithm \ref{alg: basic shampoo}.

\begin{algorithm}
    \caption{Shampoo Pseudocode (with AdaGrad Grafting) for Training MLPs}
    \label{alg: basic shampoo}
    \begin{algorithmic}
        \Require Parameters $\{W^{(i)}\}_{i = 1}^n$ with $W_0^{(i)} \equiv W^{(i)} \in \mathbb{R}^{d_i \times d_{i - 1}}$, learning rate schedule $\{\alpha_t\}_{t = 1}^T$ with $\alpha_t > 0$, epsilon for Shampoo $\epsilon > 0$, epsilon for AdaGrad $\epsilon_{\text{graft}} > 0$, maximum number of iterations $T$
        \State
        \For{$i = 1, ..., n$}
        \State Set $L_{-1}^{(i)} = \epsilon I_{d_i} \in \mathbb{R}^{d_i \times d_i}$, $R_{-1}^{(i)} = \epsilon I_{d_{i - 1}} \in \mathbb{R}^{d_{i - 1} \times d_{i - 1}}$. \Comment{Initialize Shampoo states.}
        \State Set $A_{-1}^{(i)} = 0 \in \mathbb{R}^{d_i \times d_{i - 1}}$. \Comment{Initialize AdaGrad state.}
        \EndFor
        \State
        \For{$t = 0, 1, 2, ..., T - 1$}
        \For{$i = 1, ..., n$}
        \State Compute (mini-batch) stochastic gradient $G_t^{(i)} = \nabla_{W^{(i)}} f_t(w) \in \mathbb{R}^{d_i \times d_{i - 1}}$ for parameter $i$.
        \State Update Shampoo factor matrices:
        \begin{align*}
            L_t^{(i)} & \leftarrow L_{t - 1}^{(i)} + G_t^{(i)} [G_t^{(i)}]^T, \\
            R_t^{(i)} & \leftarrow R_{t - 1}^{(i)} + [G_t^{(i)}]^T G_t^{(i)}
        \end{align*}
        \State Update AdaGrad state:
        $$A_t^{(i)} \leftarrow A_{t - 1}^{(i)} + [G_t^{(i)}]^{\odot 2}$$
        \State Compute matrix root inverses: 
        \begin{align*}
            \bar{L}_t^{(i)} & \leftarrow [L_t^{(i)}]^{-1/4}, \\
            \bar{R}_t^{(i)} & \leftarrow [R_t^{(i)}]^{-1/4}
        \end{align*}
        \State Compute Shampoo search direction:
        \begin{align*}
            P_{t, \text{Shampoo}}^{(i)} & \leftarrow \bar{L}_t^{(i)} G_t^{(i)} \bar{R}_t^{(i)} \\
            P_{t, \text{graft}}^{(i)} & \leftarrow G_t^{(i)} / ([A_t^{(i)}]^{1/2} + \epsilon_{\text{graft}} 1_{d_i} 1_{d_{i - 1}}^T) \\
            P_t^{(i)} & \leftarrow - \left\| P_{t, \text{graft}}^{(i)} \right\|_F \frac{P_{t, \text{Shampoo}}^{(i)}}{\left\| P_{t, \text{Shampoo}}^{(i)} \right\|_F}
        \end{align*}
        \State Update parameter: 
        $$W_{t + 1}^{(i)} \leftarrow W_t^{(i)} + \alpha_t P_t^{(i)}$$
        \EndFor
        \EndFor
    \end{algorithmic}
\end{algorithm}

\section{Implementation Details}
\label{sec:implementation}

Many additional improvements and heuristics are incorporated into the Shampoo optimizer implementations. Several of these heuristics have been employed in the JAX and OPTAX implementations of Shampoo and have also been incorporated into our PyTorch implementation here \cite{paszke2019pytorch, jax2018github, anil2021scalable}. We provide a high-level description of different heuristics, including using exponentially-weighted moving average estimates of the first- and second-moments, weight decay, momentum and Nesterov acceleration, and the exponent multiplier and override options. A complete description of the algorithm including learning rate grafting, the main heuristics, and the main distributed memory/computation performance optimization is provided in Algorithm \ref{alg: complete shampoo}. (We ignore merging and blocking here.)

\begin{algorithm}
    \caption{Complete Distributed Shampoo Pseudocode (on Worker $j$)}
    \label{alg: complete shampoo}
    \footnotesize
    \begin{algorithmic}
        \Require Parameters $\{W^{(i)}\}_{i = 1}^n$ with $W_0^{(i)} \equiv W^{(i)} \in \mathbb{R}^{d_i \times d_{i - 1}}$, learning rate schedule $\{\alpha_t\}_{t = 1}^T$ with $\alpha_t > 0$, exponential moving average weights $\beta_1 \in [0, 1)$, $\beta_2 \in (0, 1]$, momentum $\mu > 0$, weight decay $\lambda \geq 0$, period for computing root inverse \texttt{precondition\_frequency}, initial iteration for using Shampoo preconditioning \texttt{start\_preconditioning\_step}, grafting method, maximum number of iterations $T$, flag for bias correction \texttt{use\_bias\_correction}, flag for decoupled weight decay \texttt{use\_decoupled\_weight\_decay}, number of workers $J$, number of workers per group $J_G$
        \State
        \State Assign preconditioners to different workers using a greedy method $I_1, I_2, ..., I_J \subset [n]$ based on $d_0, d_1, ..., d_n$.
        \For{$i \in I_j$}
        \State Set $L_{-1}^{(i)} = 0 \in \mathbb{R}^{d_i \times d_i}$, $R_{-1}^{(i)} = 0 \in \mathbb{R}^{d_{i - 1} \times d_{i - 1}}$. \Comment{Initialize Shampoo states.}
        \State Set $\tilde{G}_{-1}^{(i)} = 0 \in \mathbb{R}^{d_i \times d_{i - 1}}$ if $\beta_1 > 0$, $M_{-1}^{(i)} = 0 \in \mathbb{R}^{d_i \times d_{i - 1}}$ if $\mu > 0$. \Comment{Initialize additional states.}
        \State Set $A_{-1}^{(i)} = 0 \in \mathbb{R}^{d_i \times d_{i - 1}}$ (if necessary). \Comment{Initialize grafting state (if necessary).}
        \EndFor
        \State
        \For{$t = 0, 1, 2, ..., T - 1$}
        \State Compute (mini-batch) stochastic gradient $G_t^{(i)} = \nabla_{W^{(i)}} f_t(w) \in \mathbb{R}^{d_i \times d_{i - 1}}$ for all parameters $i$ in DDP fashion.
        \For{$i \in I_j$}
        \If{$\lambda > 0$ and not \texttt{use\_decoupled\_weight\_decay}} \Comment{Incorporate $\ell_2$-regularization.}
        \State $G_t^{(i)} \leftarrow G_t^{(i)} + \lambda W_t^{(i)}$
        \EndIf
        \If{$\beta_2 < 1$} \Comment{Update Shampoo factor matrices.}
        \State $L_t^{(i)} \leftarrow \beta_1 L_{t - 1}^{(i)} + (1 - \beta_1) G_t^{(i)} [G_t^{(i)}]^T$
        \State $R_t^{(i)} \leftarrow \beta_1 R_{t - 1}^{(i)} + (1 - \beta_1) [G_t^{(i)}]^T G_t^{(i)}$
        \Else
        \State $L_t^{(i)} \leftarrow L_{t - 1}^{(i)} + G_t^{(i)} [G_t^{(i)}]^T$
        \State $R_t^{(i)} \leftarrow R_{t - 1}^{(i)} + [G_t^{(i)}]^T G_t^{(i)}$
        \EndIf
        \State $A_t^{(i)} \leftarrow \texttt{UpdateGraftingState}(A_{t - 1}^{(i)}, G_t^{(i)})$ \Comment{Update grafting method's state (if necessary).}
        \If{$t \geq \texttt{start\_preconditioning\_step}$ and $t ~ \% ~ \texttt{precondition\_frequency} = 0$}
        \State $\bar{L}_t^{(i)} \leftarrow \texttt{ComputeMatrixRootInverse}(L_t^{(i)}, \epsilon, t, \texttt{use\_bias\_correction})$ 
        \State $\bar{R}_t^{(i)} \leftarrow \texttt{ComputeMatrixRootInverse}(R_t^{(i)}, \epsilon, t, \texttt{use\_bias\_correction})$
        \EndIf
        \If{$\beta_1 > 0$} \Comment{Compute filtered/exponential moving averaged gradient.}
        \State $\tilde{G}_t^{(i)} \leftarrow \beta_1 \tilde{G}_{t - 1}^{(i)} + (1 - \beta_1) G_t^{(i)}$
        \EndIf
        \State $P_{t, \text{graft}}^{(i)} \leftarrow \texttt{ComputeGraftingDirection}(\tilde{G}_t^{(i)}, t, \texttt{use\_bias\_correction})$ \Comment{Compute grafting direction.}
        \If{$t \geq \texttt{start\_preconditioning\_step}$} \Comment{Compute scaled Shampoo direction.}
        \State $P_t^{(i)} \leftarrow - \left\|P_{t, \text{graft}}^{(i)} \right\|_F \frac{\bar{L}_t^{(i)} \tilde{G}_t^{(i)} \bar{R}_t^{(i)}}{\|\bar{L}_t^{(i)} \tilde{G}_t^{(i)} \bar{R}_t^{(i)}\|_F}$
        \Else \Comment{Use grafting search direction.}
        \State $P_t^{(i)} \leftarrow P_{t, \text{graft}}^{(i)}$
        \EndIf
        \If{$\lambda > 0$ and \texttt{use\_decoupled\_weight\_decay}} \Comment{Incorporate decoupled weight decay.}
        \State $P_t^{(i)} \leftarrow P_t^{(i)} + \lambda W_t^{(i)}$
        \EndIf
        \If{$\mu > 0$} \Comment{Incorporate momentum.}
        \State $M_t^{(i)} \leftarrow \mu M_t^{(i)} + P_t^{(i)}$
        \If{\texttt{use\_nesterov}}
        \State $P_t^{(i)} \leftarrow \mu M_t^{(i)} + P_t^{(i)}$
        \Else
        \State $P_t^{(i)} \leftarrow M_t^{(i)}$
        \EndIf
        \EndIf
        \EndFor
        \State $\{P_t^{(i)}\}_{i = 1}^n \leftarrow \texttt{AllGather}( \{P_t^{(i)}\}_{i \in I_j})$
        \State $W_{t + 1}^{(i)} \leftarrow W_t^{(i)} + \alpha_t P_t^{(i)}$ for all $i = 1, ..., n$. \Comment{Update parameters.}
        \EndFor
    \end{algorithmic}
\end{algorithm}

\subsection{Training Heuristics}

In this subsection, we describe some of the additional heuristics that are commonly used with deep learning optimizers and that have been enabled with Shampoo and layer-wise learning rate grafting. When possible, we provide intuition for each heuristic.

\subsubsection{First and Second Moment Estimation}

It is common to use gradient filtering or exponential moving averages of the ``first moment'' of the gradient. This has been widely interpreted as the natural extension of momentum to adaptive gradient methods, and has been demonstrated to be useful for deterministic nonsmooth optimization as well as deep learning training; see \cite{boyd2003subgradient, kingma2014adam}. More specifically, we can filter the gradient estimator $\tilde{G}_t^{(i)}$ via exponential moving averaging and use this in place of $G_t^{(i)}$ where
\begin{equation}
    \tilde{G}_t^{(i)} = \beta_1 \tilde{G}_{t - 1}^{(i)} + (1 - \beta_1) G_t^{(i)},
\end{equation}
with $\tilde{G}_{-1}^{(i)} = 0$. When grafting, the grafted method's state is updated using the original stochastic gradient $G_t^{(i)}$, but the search direction is computed based on the filtered gradient $\tilde{G}_t^{(i)}$ instead.

One can similarly apply exponential moving averages of the Shampoo approximation for matrices as follows:
\begin{align}
    L_{t}^{(i)} & = \beta_2 L_{t - 1}^{(i)} + (1 - \beta_2) G_t^{(i)} [G_t^{(i)}]^T, \\
    R_{t}^{(i)} & = \beta_2 R_{t - 1}^{(i)} + (1 - \beta_2) [G_t^{(i)}]^T G_t^{(i)},
\end{align}
with $L_{-1}^{(i)} = 0$ and $R_{-1}^{(i)} = 0$. A similar modification can be made for Shampoo for general tensors. A bias correction term can be employed by setting $\hat{G}_t^{(i)} = \tilde{G}_t^{(i)} / (1 - \beta_1^{t + 1})$, $\hat{L}_t^{(i)} = L_t^{(i)} / (1 - \beta_2^{t + 1})$, $\hat{R}_t^{(i)} = R_t^{(i)} / (1 - \beta_2^{t + 1})$, etc. Bias correction can be interpreted either as an implicit modification to the learning rate schedule or an approach to ensure that the statistical estimate is unbiased, particularly when only a few updates of the exponential moving average have been performed; see \cite{kingma2014adam}.\\

\noindent \textbf{Usage:} To use exponential moving averaging of these quantities, one should set \texttt{betas = (beta1, beta2)} to the desired values. If \texttt{beta2 = 1}, then the implementation will use the standard summation. To enable bias correction, simply set the flag \texttt{use\_bias\_correction = True}. (This is enabled by default.)

\subsubsection{$\ell_2$-Regularization and (Decoupled) Weight Decay}

There are two variants of regularization that we have enabled: (1) standard $\ell_2$ regularization and (2) decoupled weight decay. Weight decay is sometimes used to refer to appending an L2-regularization term to the training loss function, i.e.,
\begin{equation}
    \min_{w \in \mathbb{R}^d} \mathbb{E}_{(x, y) \sim \mathcal{D}}\left[\ell(m(w; x); y) + \frac{\lambda}{2} \|w\|^2\right].
\end{equation}
From an implementation perspective, this modifies the gradient by adding an additional regularization term: $G_t^{(i)} \leftarrow G_t^{(i)} + \lambda W_t^{(i)}$ for all $i$. Notice that this impacts all aspects of the optimizer, including the gradients used in the updates of the Shampoo preconditioners and grafting method.

On the other hand, weight decay as originally introduced in \citet{hanson1988comparing} --- now commonly referred to as \emph{decoupled} weight decay following \citet{loshchilov2017decoupled} --- is not equivalent to $\ell_2$ regularization in general.\footnote{It is equivalent to $\ell_2$-regularization when using SGD through a reparameterization~\cite{loshchilov2017decoupled}.} Decoupled weight decay involves a modification of the training algorithm outside of the preconditioner. More precisely, it is defined as:
\begin{align}
    W_{t + 1}^{(i)} & = (1 - \alpha_t \lambda) W_t^{(i)} - \alpha_t P_t^{(i)} \\
    & = W_t^{(i)} - \alpha_t (P_t^{(i)} + \lambda W_t^{(i)})
\end{align}
for $i = 1, ..., n$. This method has been interpreted as a first-order approximation to a proximal method for enforcing $\ell_2$-regularization that is scale-invariant, i.e., the method remains the same even when the objective function is multiplied by some positive constant and eases hyperparameter tuning \cite{zhuang2022understanding}.

Decoupled weight decay is often implemented as a separate transformation of the parameters unless combined with momentum, as we see below. In our experiments, we found that decoupled weight decay is more effective in obtaining solutions with better generalization. Decoupled weight decay is also implemented independent of learning rate grafting, that is, 
\begin{equation}
    W_{t + 1}^{(i)} = (1 - \alpha_t \lambda) W_t^{(i)} - \alpha_t \left\|P_{t, \text{graft}}^{(i)} \right\|_F \frac{P_{t, \text{Shampoo}}^{(i)}}{\left\|P_{t, \text{Shampoo}}^{(i)} \right\|_F}.
\end{equation}

\sloppypar\noindent \textbf{Usage:} To use weight decay with parameter $\lambda$, one should set the argument \texttt{weight\_decay}. To toggle between decoupled weight decay and $\ell_2$ regularization, one can use the \texttt{use\_decoupled\_weight\_decay} flag, which is \texttt{True} by default.

\subsubsection{Momentum and Nesterov Acceleration}

For some applications, momentum and Nesterov acceleration are imperative for achieving good generalization performance and have been successfully employed with the Shampoo optimizer \cite{anil2021distributed}. This differs from first-moment estimation or gradient filtering in its functional form through its aggregation, not of the gradients, but of the search direction of the algorithm. In particular, given the Shampoo search direction $P_t^{(i)}(w_t)$ for layer $i$ at weights $w_t$, the momentum update is defined as:
\begin{align}
    M_t^{(i)} & = \mu_t M_{t - 1}^{(i)} + P_t^{(i)}(w_t) \label{eq:momentum1}\\
    W_{t + 1}^{(i)} & = W_t^{(i)} - \alpha_t M_t^{(i)}, \label{eq:momentum2}
\end{align}
with (potentially iterate-dependent) momentum parameter $\mu_t > 0$. Normally in practice, $\mu_t = 0.9$ is fixed over all iterations.

Similarly, what is known as \textit{Nesterov momentum} or \textit{Nesterov acceleration} applies a momentum-like term a second time within the update:
\begin{align}
    M_t^{(i)} & = \mu_{t - 1} M_{t - 1}^{(i)} + P_t^{(i)}(w_t) \label{eq:nesterov1} \\
    W_{t + 1}^{(i)} & = W_t^{(i)} - \alpha_t (\mu_t M_t^{(i)} + P_t^{(i)}(w_t)). \label{eq:nesterov2}
\end{align}
While momentum and Nesterov acceleration are widely used in conjunction with SGD, momentum and Nesterov acceleration methods are misnomers given that they arise from methods for minimizing deterministic quadratic functions and strongly convex functions with Lipschitz continuous gradients, with a specialized choice of $\mu_t$. These intuitions and approximations do \textit{not} necessarily hold in the stochastic regime. We provide an alternative interpretation here, building on \cite{defazio2020momentum} that re-interprets the methods as \textit{stochastic primal iterate averaging} \cite{tao2018primal}. 

In particular, one can show that the momentum method \eqref{eq:momentum1}-\eqref{eq:momentum2} is equivalent to the iteration:
\begin{align}
    Z_{t + 1}^{(i)} & = Z_t^{(i)} - \eta_t P_t^{(i)}(w_t) \label{eq:iteravg1}\\
    W_{t + 1}^{(i)} & = c_t W_t^{(i)} + (1 - c_t) Z_{t + 1}^{(i)} \label{eq:iteravg2}
\end{align}
for $c_t \in (0, 1)$ and $\eta_t > 0$. This is similar to exponential moving averaging applied on the weights, a close variant of Polyak-Ruppert averaging \cite{polyak1992acceleration} and stochastic weight averaging \cite{izmailov2018averaging}. Rather than generating a sequence of averaged weights independent of the original sequence, this algorithm uses the intermediate averaged weights to determine the search direction at each step. Similarly, one can show that the Nesterov accelerated method \eqref{eq:nesterov1}-\eqref{eq:nesterov2} is equivalent to the iteration:
\begin{align}
    Z_{t + 1}^{(i)} & = Z_t^{(i)} - \eta_t (P_t^{(i)}(w_t) + \mu_t (P_t^{(i)}(w_t) - P_{t - 1}^{(i)}(w_{t - 1})) \label{eq:iteravg3} \\
    W_{t + 1}^{(i)} & = c_t W_t^{(i)} + (1 - c_t) Z_{t + 1}^{(i)}. \label{eq:iteravg4}
\end{align}
A formal proof for both of these equivalences is provided in Appendix \ref{app: proofs}.

This interpretation provides a principled approach for incorporating weight decay and gradient filtering into momentum and Nesterov acceleration appropriately - momentum should be applied on top of all changes to the update to the parameters, including the filtered gradient and weight decay terms. Because of this interpretation, while gradient filtering and momentum may appear similar on the surface, they should be viewed as \textit{orthogonal} changes to the algorithm, and therefore we have included both options in our implementation. In addition, this technique can be used even when changing between different search directions, as it is primarily incorporating a form of iterate averaging; this motivates our design choice of using a consistent momentum term for both the grafting method and Shampoo when incorporating an initial grafting warmup phase. \\

\noindent \textbf{Usage:} To enable momentum, simply set \texttt{momentum} to a positive number; 0.9 or 0.5 is a common setting. To toggle Nesterov acceleration, set the Boolean variable \texttt{use\_nesterov}.

\subsubsection{Exponent Override and Exponent Multiplier}

Consistent with \cite{anil2021scalable}, we allow the user to modify the exponent used for Shampoo through two options: \texttt{exponent\_override} and \texttt{exponent\_multiplier}. These two options correspond to the following change:
\begin{equation}
    W_{t + 1} = W_t - \alpha_t L_t^{-\eta / p} G_t R_t^{-\eta / p}
\end{equation}
where $\eta > 0$ is the exponent multiplier and $p \in \mathbb{N}$ corresponds to the exponent override. Note that $p$ will override the standard root of $2 \omega$, where $\omega$ is the order of the tensor parameter. We have found that using either an exponent override of $p = 2$ or exponent multiplier of $\eta = 1.82$ is often effective in practice for training networks dominated by fully-connected linear layers; for more details as to why this may be the case, see Shampoo's relationship with AdaFactor in Appendix \ref{app: relationships}. \\

\noindent \textbf{Usage:} To enable, one can set \texttt{exponent\_override} as any integer and \texttt{exponent\_multiplier} as a positive number. 

\subsection{Numerical Considerations}

When implementing Shampoo, one must consider how to efficiently and accurately compute the root inverse of the factor matrices. If the root inverse is computed too inaccurately, the resulting search directions may not even be guaranteed to be descent directions in expectation! However, computing the root inverse with unnecessarily high accuracy may significantly slow down the iteration. In this subsection, we consider \emph{how} to compute the root inverse of the preconditioner as well as describe the empirical impact of numerical precision on the matrix root inverse computation.

\subsubsection{Matrix Root Inverse Solvers}

We describe different approaches that have been implemented for computing the root inverse of each factor matrix. As noted above, all factor matrices $L, R$ are symmetric positive semi-definite by definition, and we want to compute the $p$-th inverse root of $L^{-1/p}, R^{-1/p}$. By default, our implementation uses the symmetric eigendecomposition approach.

\begin{enumerate}
    \item \textit{Symmetric Eigendecomposition}: Since the factor matrices for each block preconditioner are symmetric, we can apply the symmetric eigendecomposition solver \texttt{torch.linalg.eigh} to compute the eigendecomposition for each preconditioner. In particular, ignoring the iteration number, let $L = Q_L \Lambda_L Q_L^T$ and $R = Q_R \Lambda_R Q_R^T$ be the eigendecompositions for $L$ and $R$, respectively, where $\Lambda_L, \Lambda_R$ are diagonal matrices consisting of their eigenvalues and $Q_L, Q_R$ are orthogonal matrices consisting of their eigenvectors. The standard approach for computing the root inverse is to compute the root inverse of the eigenvalues and reconstruct the matrix by taking the root inverse of their eigenvalues, i.e., $L^{-1/p} = Q_L \Lambda_L^{-1/p} Q_L^T$ and $R^{-1/p} = Q_R \Lambda_R^{-1/p} Q_R^T$. As expected, the computational cost of computing the matrix root inverse using a symmetric eigendecomposition is $O(n^3)$. \\
    
    In the presence of small positive or zero eigenvalues, numerical errors may cause some of the eigenvalues returned by the eigendecomposition solver to be negative. This is problematic since we cannot take the root inverse of a negative eigenvalue. Although one can add a multiple $\epsilon I$ of the identity, it is not clear how large to choose $\epsilon$ to avoid this problem. For this reason, we incorporate a heuristic to ensure that all eigenvalues are sufficiently positive. \\
    
    The heuristic is detailed as follows:\\
    
    \begin{mdframed}[backgroundcolor=black!10, linewidth=1pt]
    \textbf{Symmetric Eigendecomposition Approach for Computing Root Inverse}\\
    Given $L \in \mathbb{R}^{n \times n}$ (or $R$), perturbation $\epsilon > 0$, and desired exponent $r$.
    \begin{enumerate}
        \item Compute symmetric eigendecomposition $\lambda, Q \leftarrow \texttt{eigh}(L)$ where $\lambda \in \mathbb{R}^n$ and $Q \in \mathbb{R}^{n \times n}$.
        \item Compute $\lambda_{\min} \leftarrow \min_i \lambda_i$. 
        \item Compute $\lambda_{new} \leftarrow \lambda - \min(\lambda_{\min}, 0) 1 + \epsilon 1$.
        \item Form and return matrix root inverse $L_{inv} \leftarrow Q \diag(\lambda_{new}^{-r}) Q^T$. \\
    \end{enumerate}
    \vspace{-10pt}
    \end{mdframed}
    
    \vspace{10pt}
    
    We found this approach to be stable for even small choices of epsilon, such as $\epsilon = 10^{-12}$, as suggested in previous implementations. \\

    \item \textit{Coupled Inverse Newton Iteration}: Rather than employing a direct approach that decomposes the factor matrices and constructs the root inverse, we can instead consider \textit{iterative methods} to compute the root inverse. The coupled inverse Newton iteration is one such stable variant of Newton's method that requires an appropriate initialization of the matrix in order to guarantee convergence \cite{higham2008functions}. If we are interested in computing the matrix root inverse of $L \in \mathbb{R}^{n \times n}$, the coupled inverse Newton iteration is defined as follows:
    \begin{align}
        & & X_{k + 1} & = X_k \left(\frac{(p + 1) I - M_k}{p} \right), & X_0 & = \frac{1}{c} I, & & \\
        & & M_{k + 1} & = \left(\frac{(p + 1) I - M_k}{p} \right)^p M_k, & M_0 & = \frac{1}{c^p} L, & &
    \end{align}
    where $c \in \mathbb{R}$ determines the initialization. Assuming proper initialization, one expects $X_k \rightarrow L^{-1/p}$ and $M_k \rightarrow I_n$ in $O(n^3)$ FLOPs.  \\
    
    In order to guarantee convergence of the algorithm (see Theorem 7.12 in \cite{higham2008functions}), we must establish that all the eigenvalues are contained in the interval $[0, (p + 1)c^p)$. Since $\lambda(L) \in (0, \|L\|_2]$, it is sufficient to choose $c$ such that $\|L\|_2 < (p + 1) c^p$. Note that $\|L\|_2$ is expensive to compute, so we can instead bound $\|L\|_2 \leq \|L\|_F$ and require $\|L\|_F < (p + 1) c^p$. Therefore, we must have $c > \left(\frac{\|L\|_F}{p + 1} \right)^{1/p}$. One practical choice of $c$ is $c = \left(\frac{2 \|L\|_F}{p + 1} \right)^{1/p}$. \\
    
    To terminate the algorithm, we use the termination criterion based on $M_k$ as suggested by \cite{higham2008functions}:
    \begin{equation}
        \|M_k - I\|_{\infty} < \texttt{TOL}
    \end{equation}
    for some tolerance $\texttt{TOL} > 0$. By default, we set $\texttt{TOL} = 10^{-6}$. Note that this does not support the exponent multiplier option.
\end{enumerate}

\noindent Alternative solvers for efficiently computing matrix root inverses is an active area of research (see \cite{song2022fast, shumeli2022low, fasi2023computing}), and is left for future investigation.

\subsubsection{Precision for the Accumulation and Root Inverse Computation}

It is common to use low precision (\texttt{FP16}, \texttt{BFLOAT16}, \texttt{FP8}) in the forward and backward passes to compute the gradients. However, in order to ensure that we have sufficient accuracy in the matrix root inverse computation, we accumulate the factor matrices in \texttt{FP32} or \texttt{FP64} precision. With the symmetric eigendecomposition approach, we have found that using \texttt{FP32} is sufficient, although the expected accuracy of the computation depends on the condition number as well as the gaps between consecutive eigenvalues \cite{golub2013matrix}. Therefore, the choice of precision may be model-dependent based on the eigenvalue spectrum of each factor matrix for each parameter. 

\subsubsection{Guarding Against Eigendecomposition Failures}
\label{sec:guard}

In order to protect against catastrophic failure of the \texttt{torch.linalg.eigh} kernel when applied to certain edge cases, we have enabled a retry mechanism with different precisions. The logic works as follows:
\begin{enumerate}
    \item Attempt to compute \texttt{eigh(L)} in chosen precision (typically, \texttt{FP32}). If successful, continue.
    \item Attempt to compute \texttt{eigh(L.double())} in double precision. If successful, continue.
    \item Otherwise, skip computation and proceed with previously computed matrix root inverse.
\end{enumerate}

\noindent \textbf{Usage:} The guarding mechanism is enabled by default through the flag \texttt{use\_protected\_eigh}.

\section{Memory and Performance Optimizations}
\label{sec:performance}

In this section, we describe some of the memory and performance optimizations to improve both the memory footprint and speed (or wall-clock-time-per-iteration) of the algorithm. We focus primarily on optimizing for GPU architectures, although CPU architectures are also supported by our implementation.

\subsection{Distributed Memory and Preconditioner Computation}

While the Shampoo algorithm has been demonstrated to be more efficient than diagonal adaptive gradient methods at minimizing the objective function per-iteration, the additional FLOPs introduced by matrix multiplications (in lieu of element-wise multiplication) and passes to memory for intermediate buffer reads slow down the per-iteration wall-clock time. If we operate under the standard distributed data-parallel regime, where the optimizer step is replicated across all workers, each step of Shampoo will be  slower.\footnote{In the case of some large-scale models, each step could be potentially even 50-75\% slower than standard diagonal adaptive gradient methods!} An ideal practical implementation of Shampoo should have the cost of each iteration be as efficient as diagonal-scaling-based adaptive gradient methods.

In order to reduce the memory footprint and improve the computational efficiency and systems utilization of our implementation, we propose to distribute the preconditioners and their associated compute across all workers, similar to \cite{rajbhandari2020zero}. In particular, we will assign each preconditioner (including its factor matrices $L$, $R$ and its grafting state $A$) to only one or a small subset of workers. Each worker will only be responsible for computing the matrix multiplications required for maintaining its assigned preconditioners' optimizer states, as well as the corresponding part of the global search direction. After each preconditioned search direction is computed, we perform an \allgather so that all workers have the search directions for all parameters, and then they apply the parameter updates. An additional sufficiently sized buffer is required for this communication.

The pseudocode for this optimization is detailed in Algorithm \ref{alg: complete shampoo}. Figure \ref{fig:num gpus per group} shows how a single Shampoo \texttt{step} is distributed and communicated with this optimization. We detail how we implemented this optimization further below.

\begin{figure}
    \centering
    \includegraphics[width=0.9\textwidth]{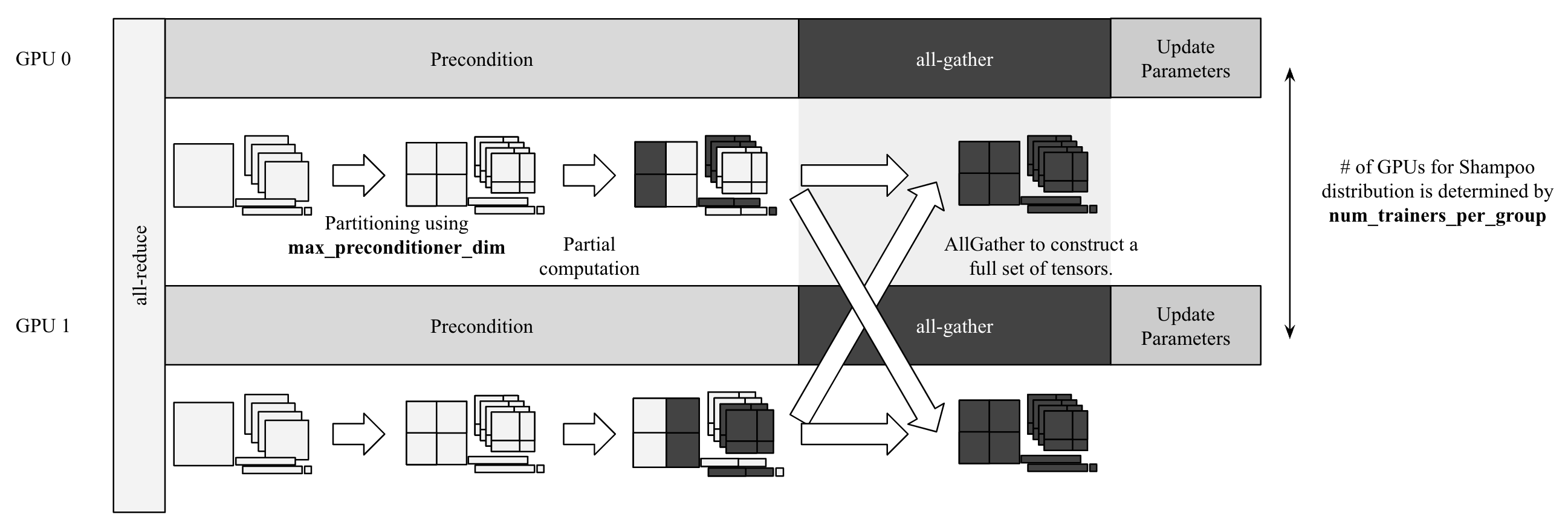}
    \caption{A single optimizer \texttt{step} of Distributed Shampoo.}
    \label{fig:num gpus per group}
\end{figure}

\subsubsection{Preconditioner Assignment and Load-Balancing via Greedy Algorithm}

In order to distribute the preconditioner memory and computation, the preconditioners need to be partitioned across all workers. Since the \texttt{AllGather} is performed on the search directions, we choose to load-balance based on its memory cost and assign preconditioners to ensure that the maximum buffer size for each worker is minimized. To do this, we employ a sorted greedy approximation algorithm as described in Algorithm \ref{alg: greedy}. The key idea is to sort the parameters based on number of variables, and assign each parameter in descending order to the worker with the fewest variables. The assignments are made prior to the instantiation of the preconditioners; see Algorithm~\ref{alg: complete shampoo}.

\begin{algorithm}
    \caption{Greedy Load-Balancing Assignment for Homogeneous Architectures}
    \label{alg: greedy}
    \begin{algorithmic}
        \Require Number of variables per parameter $d^{(1)}, ..., d^{(n)}$, total number of workers (world size) $J$, number of workers per process group $J_G$
        \State
        \State Sort the parameters such that $d^{(k_1)} \geq d^{(k_2)} \geq ... \geq d^{(k_n)}$ for $k_i \in [n]$.
        \State Initialize assignment sets $I_1 = \{\}, ..., I_J = \{\}$, where $I_j$ assigns the indexed parameters in the set to worker $j$.
        \State Initialize variable counters $C_1 = 0, ..., C_{J_G} = 0$.
        \For{$i = 1, ..., n$}
        \State Find the workers with the least variables: $\tilde{k}_i \in \arg\min_{k \in [n]} C_k$.
        \State Assign $I_{(j - 1) J_G + {\tilde{k}_i}} \leftarrow I_{(j - 1) J_G + \tilde{k}_i} \cup \{i\}$ for all $j \in [J / J_G]$.
        \EndFor
        \State Return assignments $\{I_j\}_{j = 1}^J$.
    \end{algorithmic}
\end{algorithm}

The distributed \allgather buffer will have length $Q_G \max_{j \in [n]} C_j$. In our implementation, we choose to use the \texttt{int8} data type for the distributed buffer, regardless of the precision being used. Figure \ref{fig:max buffer size} shows how using the maximum buffer size allocation may result in additional memory consumption.

\begin{figure}
    \centering
    \includegraphics[width=0.9\textwidth]{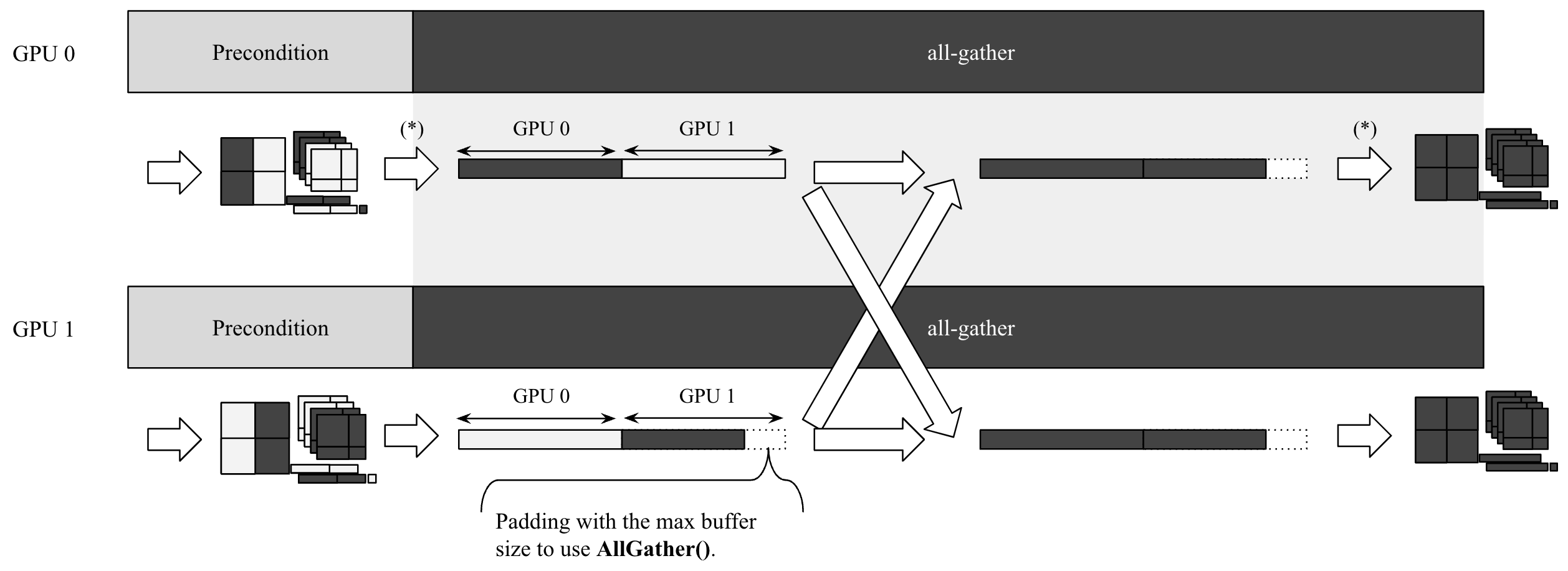}
    \caption{Maximum buffer size allocation for \texttt{AllGather} primitive.  In practice, the preconditioned tensors is a collection of \texttt{view()} of the 1D communication buffer tensor to avoid extra copy at ($\ast$) in the figure.}
    \label{fig:max buffer size}
\end{figure}

\subsubsection{Balancing Computation and Communication Through Multiple Process Groups}

As opposed to distributing the preconditioners across all workers, which may lead to high communication costs relative to the amount of compute per-worker, one can instead create multiple distinct process groups that partition the global world into smaller process groups. By distributing the computation within each process group while replicating the computation across different process groups, we can achieve more balanced compute and communication costs and observe higher performance gains. This form of hierarchical parallelism maps efficiently to the underlying systems architecture and topology as well.

We therefore provide a user parameter \texttt{num\_trainers\_per\_group} (corresponding to $Q_G$ in Algorithm \ref{alg: greedy}), which specifies the number of workers each process group should contain. Here, we assume that the user is running the algorithm on a homogeneous system architecture, where each node contains the same number of workers. In particular, we require that the \texttt{num\_trainers\_per\_group} divides the total world size with no remainder. By default, \texttt{num\_trainers\_per\_group} is equal to the number of workers per node, although this is usually not ideal for training large-scale models. 

\subsubsection{\texttt{DTensor} State Allocation}

In order to distribute the memory used by the factor matrices, grafting, momentum, and filtered gradient states, we use a new PyTorch data structure called \texttt{DTensor} (for ``Distributed Tensor''), which enhances the standard \texttt{Tensor} data structure with mesh information that describes how the tensor should be sharded or replicated across multiple workers. This enables \texttt{DTensor} to support multi-level parallelism, including various combinations of data parallelism, tensor parallelism, and pipeline parallelism. By using \texttt{DTensor}, we can specify the tensor to be replicated across only a small subset of workers, while recognizing the existence of the distributed tensor on every rank, which is necessary for creating efficient distributed checkpointing solutions. 

This solution enables us to approximately reduce the overall memory cost per-worker by a factor of $Q_G$. Note that this will depend on the quality of load-balancing, which depends on the distribution of the parameter shapes. To enable \texttt{DTensor}, one can use the \texttt{use\_dtensor} flag (this is enabled by default).

\subsection{Handling Large-Dimensional Tensors}
\label{sec:large tensors}

Shampoo significantly reduces the amount of memory required to produce a block-diagonal approximation compared to full-matrix AdaGrad. However, for tensors with large dimensions, i.e., $d_i \gg 0$ for some $i$, it is still possible for Shampoo to remain infeasible in terms of its computational and memory cost. In order to reduce memory consumption, we have enabled multiple approaches for handling large tensors consistent with those suggested in \cite{gupta2018shampoo, anil2021scalable}. We present these approaches for completeness in the order from most-to-least memory-consuming. All approaches rely on the same hyperparameter \texttt{max\_preconditioner\_dim}. The memory and computational cost of each of these approaches is summarized in Table \ref{tab:memory-computation reqs}. 

\begin{table}
    \centering
    \begin{tabular}{ccccc} \toprule
         & \multicolumn{2}{c}{Matrix ($d_1 \times d_2$)} & \multicolumn{2}{c}{Order-$\omega$ Tensor ($d_1 \times ... \times d_{\omega}$)} \\
         \texttt{LargeDimMethod} & Memory Cost & Computational Cost & Memory Cost & Computational Cost \\ \midrule
        \texttt{BLOCKING} & $4 d_1 d_2$ & $O(b^3)$ & $\frac{2 \omega}{b^{\omega - 2}} \prod_{i = 1}^\omega d_i$ & $O(b^3)$ \\
        \texttt{ADAGRAD} & $d_1 d_2$ & $O(d_1 d_2)$ & $\prod_{i = 1}^{\omega} d_i$ & $O(\prod_{i = 1}^{\omega} d_i)$ \\
        \texttt{DIAGONAL} & $d_1 + d_2$ & $O(d_1 d_2)$ & $\sum_{i = 1}^{\omega} d_i$ & $O(\prod_{i = 1}^{\omega} d_i)$ \\ \bottomrule
    \end{tabular}
    \caption{Summary of memory and computational requirements for different large-dimensional methods for matrices and general tensors. Assumes that $b$ is the block size.}
    \label{tab:memory-computation reqs}
\end{table}

\subsubsection{Merging and Blocking}

Instead of applying Shampoo to the full tensor, we can instead reshape the tensor by merging small dimensions and blocking the tensor into multiple smaller sub-tensors. On one extreme, blocking enables us to use a coarser approximation at lower memory and computational cost. It is an ideal approximation since it preserves the original tensor structure of the parameters that Shampoo relies upon for its Kronecker product approximation. On the other extreme, merging dimensions enables us to remove unnecessary dimensions and move towards using full-matrix AdaGrad for that particular parameter.

\begin{figure}
    \centering
    \includegraphics[align=c, width=0.15\textwidth]{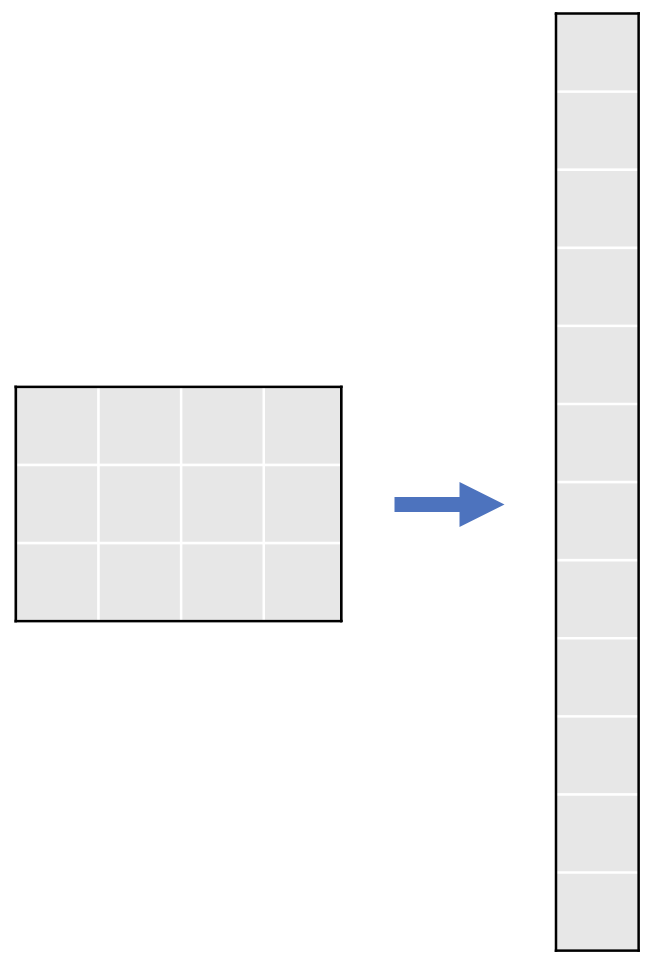} 
    \hspace{20pt}
    \includegraphics[align=c, width=0.4\textwidth]{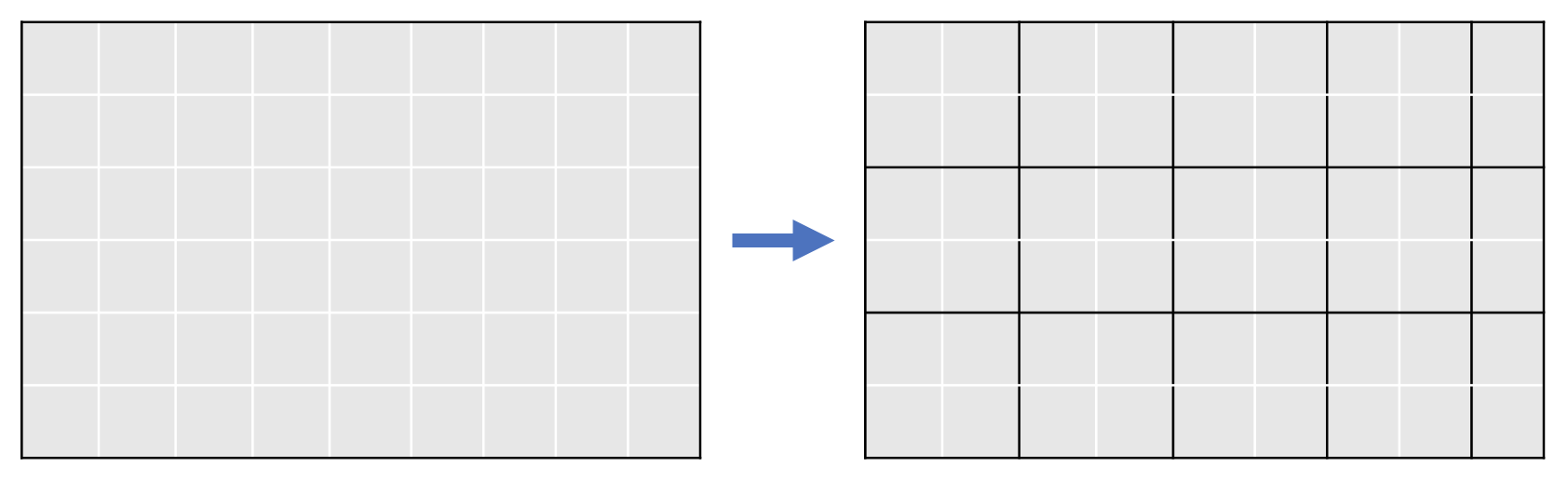}
    \caption{Picture of merging (left) and blocking (right).}
    \label{fig:merging-blocking}
\end{figure}

Merging small dimensions involves setting a maximum dimension and merging consecutive dimensions until its product exceeds the maximum dimension. For example, with maximum dimension $8$, a $10 \times 2 \times 2 \times 4$ dimensional tensor would be reshaped to $10 \times 4 \times 4$ after merging. This is particularly useful for getting rid of redundant (or unit) dimensions. We merge consecutive dimensions in order to ensure that no data movement is required and only \texttt{torch.view} is necessary to reshape the tensor. If all dimensions are merged, then Shampoo is applied to a vector, which is equivalent to applying full-matrix AdaGrad.

Blocking takes a tensor and creates multiple sub-tensors with a given block size $b$. For example, for a second-order tensor (or matrix) $W \in \mathbb{R}^{m \times n}$, we may block the matrix as:
\begin{equation*}
    W = \begin{bmatrix}
    W_{1,1} & W_{1, 2} & ... & W_{1, k_n} \\
    W_{2, 1} & W_{2, 2} & ... & W_{2, k_n} \\
    \vdots & \vdots & \ddots & \vdots \\
    W_{k_m, 1} & W_{k_m, 2} & ... & W_{k_m, k_n}
    \end{bmatrix}
\end{equation*}
where $k_m = \lceil m / b \rceil$ and $k_n = \lceil n / b \rceil$. Note that one can block such that $W_{i, j}$ are all similar in size (which is not necessarily $b \times b$) or such that $W_{i, j} \in \mathbb{R}^{b \times b}$ for $i = 1, ..., k_m - 1$ and $j = 1, ..., k_n - 1$. In our implementation, we opt for the latter in order to best exploit the GPU's capabilities. 

Shampoo is then applied to each block $W_{i, j}$. Note that this also corresponds to partitioning the factors for $W$ into smaller blocks, i.e., if $L$ and $R$ correspond to the left and right preconditioner factors for $W$, then:
\begin{equation*}
    L^{1/2} \otimes R^{1/2} \mapsto 
    P_{\pi}^T \begin{bmatrix}
    L_{1, 1}^{1/2} \otimes R_{1, 1}^{1/2} & 0 & ... & 0 \\
    0 & L_{1, 2}^{1/2} \otimes R_{1, 2}^{1/2} & ... & 0 \\
    0 & 0 & \ddots & 0 \\
    0 & 0 & ... & L_{l, k}^{1/2} \otimes L_{l, k}^{1/2}
    \end{bmatrix} P_{\pi}.
\end{equation*}
where $P_{\pi}$ is a permutation matrix that maps $w = \vect(W)^T$ to 
$$w_{\pi} = (\vect(W_{1, 1})^T, \vect(W_{1, 2})^T, ..., \vect(W_{l, k})^T)^T = P_{\pi} w.$$
We use the same block size hyperparameter, called \texttt{max\_preconditioner\_dim} in our implementation, for both merging and blocking. Merging and blocking therefore has a multi-faceted impact on model quality, memory, and performance. We summarize the impact of modifying the block size on each of these aspects below:

\begin{enumerate}
    \item \textbf{Model Quality}: As the block size increases, we expect the model quality to improve because our approximation will remove  dimensions and eventually use full-matrix AdaGrad for that parameter. This incentivizes using large block sizes as long as the factor matrices fit in memory and the algorithm's performance is not too slow. 
    \item \textbf{Memory}: For a general order-$\omega$ tensor $W \in \mathbb{R}^{d_1 \times ... \times d_\omega}$ and block size $b$ that divides $d_1, ..., d_{\omega}$, the total memory cost of blocked Shampoo is $\frac{2 \omega}{b^{\omega - 2}} \prod_{i = 1}^\omega d_i$. The factor $2$ arises because we have to store both the factor matrices and their root inverses. Note that if $\omega < 2$, then as $b$ increases, the memory cost increases. However, if $\omega > 2$, then as $b$ increases, the memory cost decreases. In the matrix case ($\omega = 2$), blocked Shampoo has constant memory cost $4 d_1 d_2$.
    \item \textbf{Performance}: Using too small of a block size can lead to high latency from increased GPU/CUDA kernel launch overheads and reduced compute efficiency. On the other hand, using too large of a block size results in large factor matrices that are costly to invert. Therefore, performance is optimized by a set of block sizes that trade off these two extremes. In our experience, using a block size $b \in \{1024, 2048, 4096, 8192\}$ is ideal for performance.
\end{enumerate}

\subsubsection{Diagonal AdaGrad Preconditioner}

Alternatively, we provide the option to use the standard diagonal AdaGrad, RMSProp, or Adam preconditioner in place of Shampoo if any of the dimensions exceeds \texttt{max\_preconditioner\_dim}. This reduces the memory cost to $d_1 d_2$ for the matrix case and $\prod_{i = 1}^{\omega} d_i$ for the general order-$\omega$ tensor case, and offers the same performance as diagonal adaptive gradient methods. In general, we expect this approach to yield model accuracies between blocked Shampoo and diagonal Shampoo.

\subsubsection{Diagonal Shampoo Preconditioner} 

Lastly, we can also diagonalize each factor matrix for dimensions larger than \texttt{max\_preconditioner\_dim}. In the two-dimensional case, this reduces to using $\tilde{L}_t = \matdiag(L_t)$ and $\tilde{R}_t = \matdiag(R_t)$. Note that this reduces the memory cost to $d_1 + d_2$ for the matrix case and $\sum_{i = 1}^{\omega} d_i$ for the general tensor case. Since the matrix is diagonal, it is not necessary to store the root inverse matrices. This approximation may be useful for very large tensors, such as embedding tables, but yields a worse approximation than diagonal AdaGrad if all dimensions are diagonalized. Diagonal Shampoo shares a close relationship with AdaFactor \cite{shazeer2018adafactor} and row-wise AdaGrad \cite{gupta2014training, mudigere2022software}; see Appendix \ref{app: relationships} for more details. 

\subsection{Periodic Root Inverse Computation}

Since the most expensive computation is the root inverse computation, one natural way of reducing the overall wall-clock time of each iteration is to only \textit{periodically} compute the matrix root inverse of the factor matrices, similar to \cite{anil2021scalable}. This is controlled by the \texttt{precondition\_frequency} hyperparameter. This speedup comes at the cost of using stale root inverse matrices, which can slow convergence and impact the final model quality achieved by the optimizer. 

Staleness can particularly have a detrimental impact on convergence at the beginning of training, when the preconditioners are less stable. For this reason, we also incorporate a hyperparameter \texttt{start\_preconditioning\_step} for delaying Shampoo preconditioning. Prior to iteration \texttt{start\_preconditioning\_step}, Distributed Shampoo will take steps using the grafted method before switching to Shampoo preconditioning (with grafting).

Both of these optimizations are consistent with \cite{anil2021scalable}. However, because we are primarily focused on supporting hardware architectures that support higher precision, we \textit{do not} offload the matrix root inverse computation to CPU.

\subsection{Comparison with JAX Implementation for TPU/CPU Architectures}

While the core algorithm and some of the performance optimizations such as merging, blocking, and the periodic computation of the matrix root inverses are shared across our PyTorch implementation and the JAX/OPTAX implementation \cite{anil2021scalable}, key framework (PyTorch vs JAX/OPTAX) and hardware architecture (homogeneous GPU and CPU architectures vs heterogeneous TPU/CPU architectures) differences lead to some critical differences between these two implementations. We discuss these differences below.

\subsubsection{CPU Offloading}

Since both GPU and CPU natively support FP32 and FP64 computation, our implementation does not offload the root inverse computation onto CPU to avoid unnecessary data movement. This contrasts with the JAX implementation for TPU/CPU architectures, which do not offer FP32 or FP64 support, and therefore makes offloading a necessity \cite{anil2021scalable}.

This specifically impacts the staleness of the root inverse matrices. While the matrix root inverses in the PyTorch implementation will be stale for up to \texttt{precondition\_frequency} iterations (before all root inverse matrices are re-computed based on the updated factor matrices), the JAX implementation will be stale for $2 \times \texttt{precondition\_frequency}$, as its offloading onto CPU and overlapping of the matrix root inverse computation on CPU with Shampoo's (stale) preconditioned updates on TPU creates two periods of staleness. 

\subsubsection{Compiler vs Hand-Optimized Kernels and Communications}

Prior to PyTorch 2.0, PyTorch did not offer a compiler that can automatically fuse operators using \texttt{torch.compile}. JAX, on the other hand, relies on XLA to compile and run NumPy programs on GPU and TPU. As a result, our PyTorch implementation requires the use of hand-optimized kernels in order to run efficiently.

One example is the use of PyTorch's \texttt{\_for\_each} operators, which fuse the element-wise operators for each parameter together. Our communications optimizations and distributed buffer instantiation are also explicitly defined, unlike in the JAX implementation. Incorporation of PyTorch 2.0 is left for future work.

\subsubsection{FP32 vs FP64 Default Precision}

Unlike \citet{anil2021scalable}, we have found that using single precision is often sufficient for our workloads, although we provide the option for the user to specify the factor matrix precision through \texttt{preconditioner\_dtype}. To further avoid failure of the eigendecomposition, we have enabled a guarding mechanism as described in Section \ref{sec:guard}.

\subsubsection{Eigendecomposition vs Coupled Inverse Newton}

By default, our implementation uses PyTorch's Hermitian/symmetric eigendecomposition operator \texttt{torch.linalg.eigh}, which internally calls \texttt{CUSOLVER}'s symmetric eigendecomposition solvers. The JAX implementation instead relies on a warm-started coupled inverse Newton, as described in \cite{anil2021scalable}.

\section{Numerical Results}
\label{sec:numerical}

In this section, we provide experimental results for training a ResNet50 model, which contains 25.5M trainable parameters, on the ImageNet-1k dataset~\cite{deng2009imagenet, he2016deep}. We compare our implementation of Distributed Shampoo against the standard baseline on this workload, SGD with Nesterov acceleration. The results demonstrate that Shampoo can provide significant reductions in overall wall-clock time and number of steps compared to a well-tuned Nesterov baseline. These reductions are observed despite the additional FLOPs incurred by Shampoo's update rule.

We concentrate on three sets of experiments:
\begin{enumerate}
    \item Comparing the performance of both methods with a fixed training budget of 90 epochs, which is the standard budget for optimal validation accuracy when training with SGD with Nesterov momentum. 
    \item Comparing the performance of both methods to achieve a given target validation accuracy, without constraining the budget of training epochs. This enables us to observe substantial savings in overall wall-clock time as well as number of epochs to achieve a fixed validation accuracy by Shampoo.\footnote{Note that an experiment with more training epochs is \textit{not equivalent} to merely continuing an experiment with fewer epochs since the learning rate schedule depends directly on the total number of steps in each experiment.}
    \item Evaluating the sensitivity of Shampoo and Nesterov to the choice of base learning rate.
\end{enumerate}
We use the notation \texttt{average} $\left[\frac{\texttt{max}}{\texttt{min}}\right]$, when referring to aggregate results across multiple seeds.

\subsection{Experimental Setup}

We use SGD with Nesterov momentum of $\mu = 0.9$ with a linear warmup-then-cosine learning rate scheduler as the baseline for our experiments. This choice of the learning rate schedule is standard practice in the computer vision community; see Section 5.1 of~\cite{he2019bag}.

To highlight the versatility of Shampoo as an enhancement to existing training pipelines, our experiments use SGD as the grafted method, matching the optimizer choice in the baseline. A comprehensive description of the choice of hyperparameters for both methods are included in Appendix \ref{sec:appx:hyperparameters}.

Much of the additional overhead incurred by Shampoo depends on two factors: (i) the computation of the inverse preconditioners, and (ii) the preconditioning of the observed gradient to obtain the Shampoo search direction. We can control (i) by amortizing the cost of the preconditioner computation across multiple iterations, using \textit{stale} preconditioners in-between updates. Moreover, the cost of the preconditioner inversion will also be governed by the maximum allowed preconditioner dimension. Notably, in the ImageNet task, Shampoo can operate effectively with a preconditioner update frequency of 50 steps, and a maximum preconditioner dimension of 2048 (beyond which blocking is applied, see \S\ref{sec:large tensors}) with minimal overhead. All experiments below use these settings. Appendix \ref{sec:appx:performance_ablations} contains ablations on the choices of the \texttt{max\_preconditioner\_dim} and \texttt{precondition\_frequency} hyperparameters.

\subsection{Results and Analysis}

\subsubsection{Fixed Epoch Budget}

\begin{figure}
\footnotesize
    \includegraphics[scale=0.6]{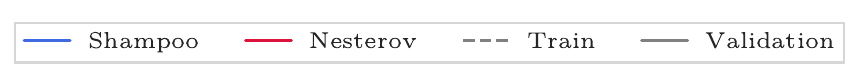} \\
    \includegraphics[scale=0.6, trim=0 30pt 0 0, clip]{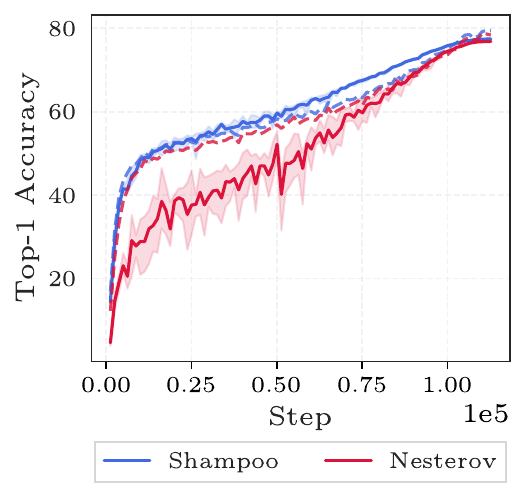}
    \hspace{0.45cm}%
    \includegraphics[scale=0.6, trim=0 30pt 0 0, clip]{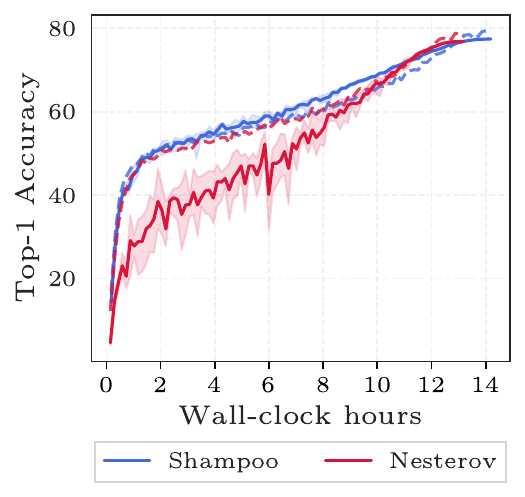} \\
    \includegraphics[scale=0.6, trim=0 30pt 0 0, clip]{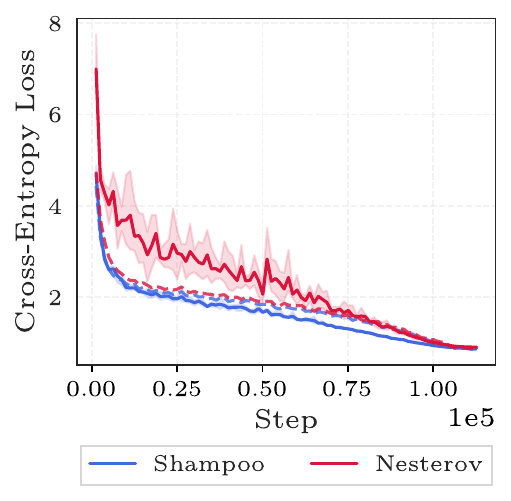}
    \hspace{0.6cm}%
    \includegraphics[scale=0.6, trim=0 30pt 0 0, clip]{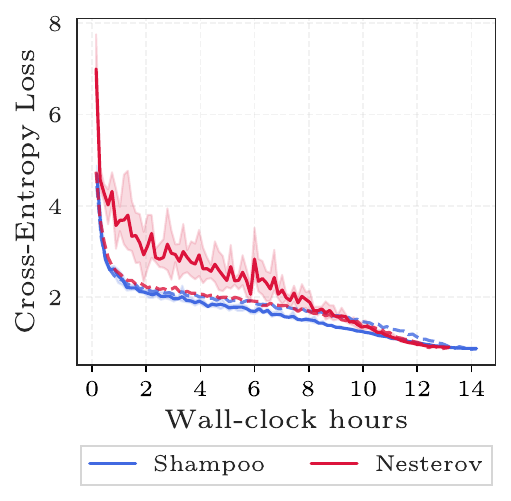}
\caption{Top-1 accuracy and cross-entropy loss on the ImageNet dataset. Shaded regions show min-max bounds across 5 seeds. Bounds on the training metrics are omitted for readability. \textbf{All throughout training, the iterates visited by Shampoo achieve better validation accuracy and less variability than those of the Nesterov baseline.}}
\label{fig:main_comparison_90_epochs}
\end{figure}

Figure \ref{fig:main_comparison_90_epochs} shows top-1 accuracy and cross-entropy loss metrics under a fixed training budget of 90 epochs. Shampoo consistently achieves better validation accuracy than Nesterov, at $77.44\% \left[\frac{77.58}{77.36}\right]$ vs $76.85 \% \left[\frac{76.93}{76.78}\right]$. The improvements in the validation metrics by Shampoo can be observed throughout training with respect to both steps and wall-clock time. Notably, the accuracy and loss measurements for Shampoo in the validation set are significantly less volatile than those of the Nesterov runs. This reduction in variability is desirable since it indicates more robust behavior of the optimizer, and makes individual runs more informative of the method's general performance. Despite the increased complexity of the update rule, using the amortization scheme above, Shampoo only incurs an 8\% wall-clock time overhead to complete 90 epochs.

We want to emphasize that, in these experiments, Shampoo is run using \textit{exactly the same hyperparameter values} as in the Nesterov training recipe with grafting from SGD (including the number of epochs, base learning rate, learning rate scheduler, and weight decay), and that these hyperparameters were specifically tuned for Nesterov. The only hyperparameter tuning we performed for Shampoo were the ablations on the hyperparameters \texttt{max\_preconditioner\_dim} and \texttt{precondition\_frequency} (see App. \ref{sec:appx:performance_ablations}) to determine an acceptable trade-off between preconditioner staleness and computational overhead.

There is also a qualitative difference in the generalization gap induced by the different optimizers throughout training. Interestingly, Shampoo produces models whose accuracy and loss track each other more closely between training and validation compared to Nesterov. This disparity is most evident at the beginning of training and is reduced in later stages. It may be precisely this closer tracking of the validation metrics that underlies Shampoo's improvements over SGD. An understanding of Shampoo's ``implicit regularization'' is left for future research.

\subsubsection{Epoch Budget Ablation}

Figure \ref{fig:epoch_sweep} displays the results of experiments with a changing training budget, between 40 and 90 epochs. Across all epoch budgets, Shampoo displays a similar reduction in the volatility of the validation metrics as discussed above, and reliably achieves better performance than Nesterov.

Figures \ref{fig:epoch_sweep} (c) and (d) show the speed-ups in terms of number of steps and wall-clock time required for Shampoo to achieve the same validation accuracy as when training with Nesterov for 90 epochs. Shampoo required 1.5x \textit{fewer} steps and 1.35x \textit{less} time than the Nesterov baseline. Similarly, Figures \ref{fig:epoch_sweep} (g) and (h) demonstrate that Shampoo yields speed-ups of 1.8x in terms of steps and 1.69x in terms of wall-clock time, to achieve the same validation loss as when training with Nesterov for 90 epochs.

\begin{figure}
\footnotesize
    \includegraphics[scale=0.6]{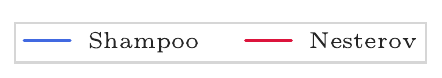} \\
    \includegraphics[scale=0.6, trim=0 30pt 0 0, clip]{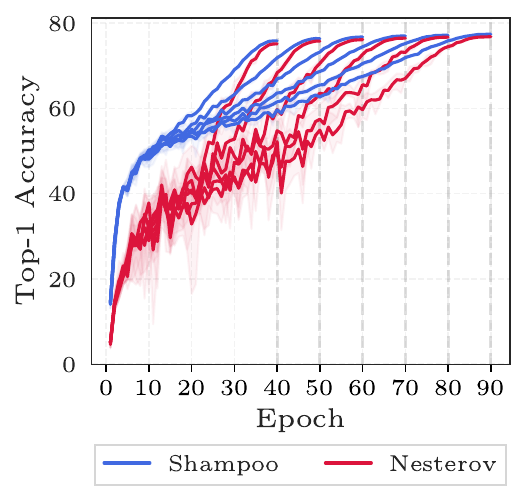}
    \hspace{0.5cm}%
    \includegraphics[scale=0.6, trim=0 30pt 0 0, clip]{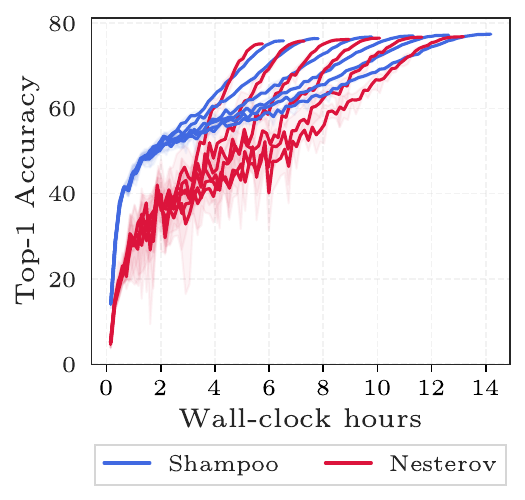} \\
    \includegraphics[scale=0.6, trim=0 30pt 0 0, clip]{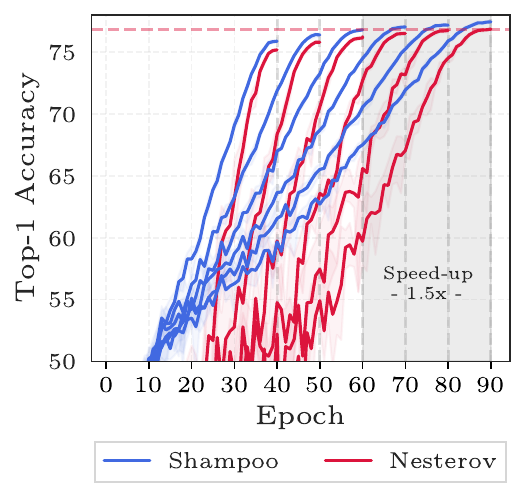}
    \hspace{0.5cm}%
    \includegraphics[scale=0.6, trim=0 30pt 0 0, clip]{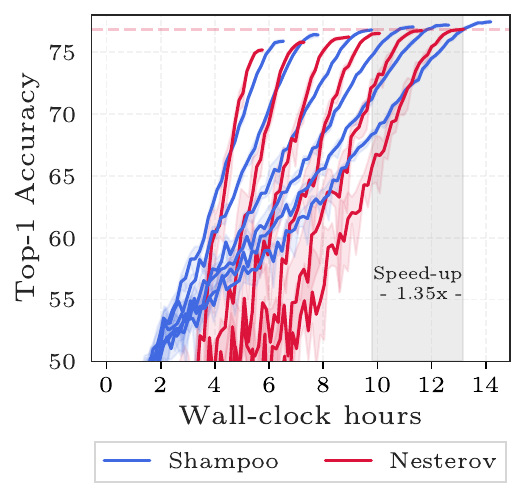} \\
    \includegraphics[scale=0.6, trim=0 30pt 0 0, clip]{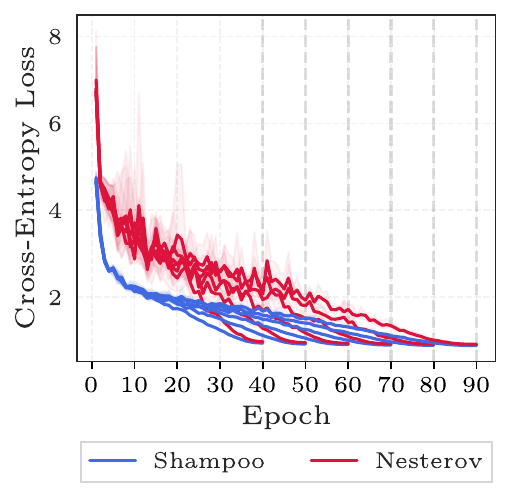} 
    \hspace{0.7cm}%
    \includegraphics[scale=0.6, trim=0 30pt 0 0, clip]{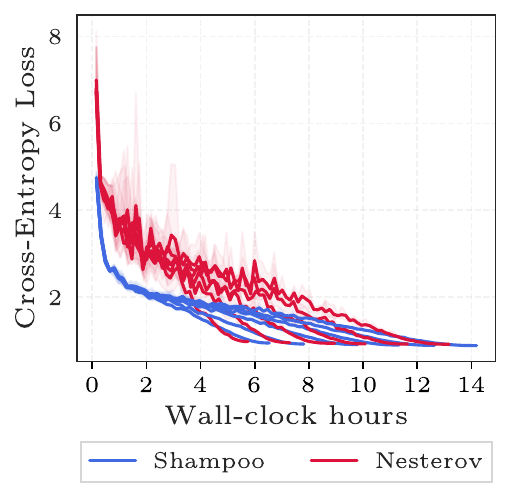}\\
    \includegraphics[scale=0.6, trim=0 30pt 0 0, clip]{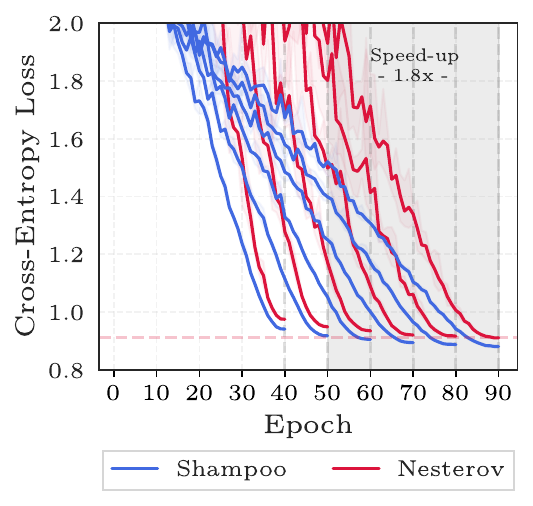} 
    \hspace{0.5cm}%
    \includegraphics[scale=0.6, trim=0 30pt 0 0, clip]{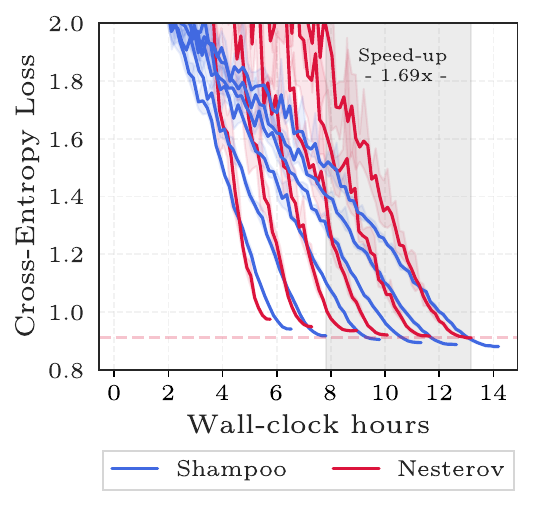}
\caption{Ablation on achieved validation performance with a changing budget of training epochs. Second and fourth rows correspond to a detail view of the first and third rows, respectively. Shaded regions show min-max bounds across 5 seeds. Training metrics are omitted for readability. \textbf{60-epoch Shampoo delivers a 1.35x reduction in terms of the wall-clock time required to achieve the validation accuracy of 90-epoch SGD.} This corresponds to a 1.5x step-wise reduction.}
\label{fig:epoch_sweep}
\end{figure}

\subsubsection{Sensitivity to the Learning Rate}

\begin{figure}
\footnotesize
    \includegraphics[scale=0.6]{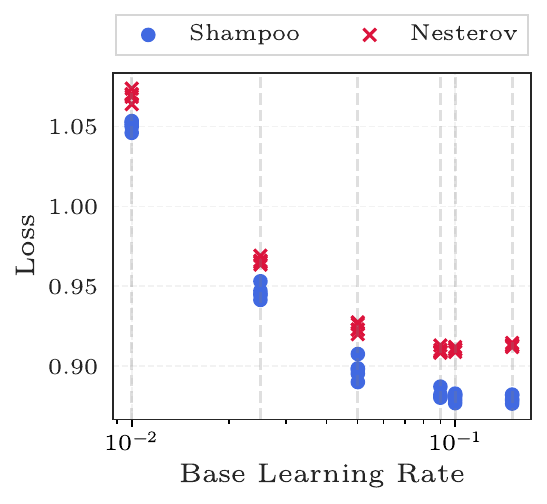}
    \hspace{0.5cm}
    \includegraphics[scale=0.6]{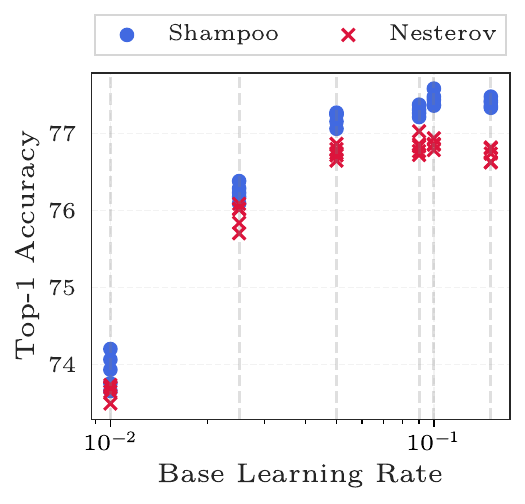}
\caption{Sensitivity of Nesterov and Shampoo to changes in the base learning rate. Plot shows metrics on the validation set, with markers indicating different seeds. \textbf{Shampoo achieves consistently better loss and accuracy than Nesterov across a wide range of choices of the base learning rate.} However, the performance of both methods is still heavily influenced by this hyperparameter.}
\label{fig:lr_sensitivity}
\end{figure}

    Figure \ref{fig:lr_sensitivity} displays the validation loss and accuracy for both methods trained using different values for the base learning rate. These experiments use a fixed budget of 90 epochs. Ideally, adaptive methods should provide better robustness to choices of certain hyperparameters like the learning rate. As seen in Figure \ref{fig:lr_sensitivity}, the performance of Shampoo is reliably superior to that of Nesterov over the tested range of learning rates. Nevertheless, different values of the learning rate lead to significant performance changes for both methods. These results indicate that, while Shampoo is a promising preconditioned gradient technique, further research is required to improve the method's robustness to hyperparameter choices.

\section{Related Work}

There has been extensive research on the design of preconditioned stochastic gradient algorithms for training deep neural networks. The stochastic gradient method initially enabled the training of machine learning models on large-scale datasets via stochastic approximation \cite{robbins1951stochastic, bottou1991stochastic, bottou2010large, lecun1998gradient}. Subsequent work focused on improving the stochastic gradient method by incorporating momentum \cite{rumelhart1986learning} and iterate averaging techniques \cite{polyak1992acceleration}. Multiple directions were subsequently pursued to improve upon the stochastic gradient method by incorporating preconditioning, both for general stochastic optimization and online convex optimization. \citet{bottou2018optimization} provides a good review of such methods. We elaborate on the three main directions next.

The first group of methods extend deterministic smooth optimization methods that utilize curvature information, such as Newton and quasi-Newton methods, for stochastic optimization. This class of methods has typically relied on diagonal approximations \cite{bordes2009sgd}, sub-sampling or sketching the Hessian \cite{xu2020newton, xu2020second, berahas2020investigation, pilanci2017newton, xu2016sub}, ensuring consistency when evaluating gradient differences \cite{berahas2016multi, schraudolph2007stochastic}, re-sampling correction pairs \cite{berahas2022quasi}, and using adaptive sampling or progressive batching, i.e., increasing the batch size via adaptive mechanisms \cite{bollapragada2018adaptive, bollapragada2018progressive, devarakonda2017adabatch}. These methods were investigated in the context of deep learning by \cite{martens2010deep, martens2012training, martens2011learning}. Most recently, Kronecker product approximations have also been applied to quasi-Newton methods through the development of the K-BFGS algorithm \cite{goldfarb2020practical}.

The second group of methods extend the natural gradient method \cite{amari1998natural} for training neural networks. The natural gradient method has been shown to be equivalent to the generalized Gauss-Newton method \cite{martens2020new, kunstner2019limitations} and has been further analyzed in \citet{zhang2019fast}. K-FAC was the first method to propose using block-diagonal preconditioners with Kronecker product approximations for training neural networks \cite{martens2015optimizing}. This method, which was built on top of the natural gradient method, was extended to different layer types and distributed setups; see \cite{grosse2016kronecker, ba2017distributed, martens2018kronecker, george2018fast}. Alternatives that extend the natural gradient method such as TNT have also been proposed in \citet{ren2021tensor}.

Lastly, a class of preconditioned gradient methods, known as adaptive gradient methods, preconditioned the (stochastic) gradient by the accumulation of outer products of the observed gradients; see \cite{duchi2011adaptive}. Originally designed for online (possibly nonsmooth) convex optimization, the diagonally approximated forms of these methods have gained wide traction in the deep learning community. Subsequent works extended these methods by incorporating other heuristics, such as gradient filtering, decoupled weight decay, and block re-scalings; see \cite{kingma2014adam, loshchilov2017decoupled, you2019large, you2018imagenet}. The work on (Distributed) Shampoo exploited Kronecker product approximations akin to K-FAC to design specialized adaptive gradient method for training deep networks \cite{gupta2018shampoo, anil2021scalable}. 

In terms of performance optimizations, our implementation shares the most similarities with DeepSpeed's ZeRO-1 optimizer, which shards the optimizer states to optimize memory for large-scale models \cite{rajbhandari2020zero, rajbhandari2021zero}. Our performance optimizations can also be interpreted as using solely the optimizer portion of PyTorch's Fully Sharded Data Parallel (FSDP) and Hybrid Sharded Data Parallel (HSDP) \cite{zhao2023pytorch}.

\begin{acks}
We thank Rohan Anil and Vineet Gupta for the original development of the Distributed Shampoo algorithm, its implementation in JAX, and their suggestions. We also thank Simon Lacoste-Julien for his support of this work. 

We thank Adnan Aziz, Malay Bag, Pavan Balaji, Xiao Cai, Shuo Chang, Nikan Chavoshi, Wenlin Chen, Xi Chen, Ching-Hsiang Chu, Weiwei Chu, Aaron Defazio, Alban Desmaison, Quentin Duval, Assaf Eisenman, Zhuobo Feng, Leon Gao, Andrew Gu, Yizi Gu, Yuchen Hao, Tao Hu, Yusuo Hu, Yuxi Hu,  Jianyu Huang, Minhui Huang, Shakti Kumar, Ming Liang, Mark Kim-Mulgrew, Guna Lakshminarayanan, Ming Liang, Wanchao Liang, Xing Liu, Ying Liu, Liang Luo, Yinbin Ma, Wenguang Mao, Maxim Naumov, Jongsoo Park, Yi Ren, Ke Sang, Xinyue Shen, Min Si, Dennis van der Staay, Ping Tak Peter Tang, Fei Tian, James Tian, Andrew Tulloch, Sanjay Vishwakarma, Ellie Wen, Lin Xiao, Shawn Xu, Ye Wang, Chunzhi Yang, Jiyan Yang, Lin Yang, Chunxing Yin, Christina You, Jiaqi Zhai, Susan Zhang, Zhang Zhang, Gedi Zhou, and Wang Zhou for their excellent internal contributions, support, feedback, and backing of this work.
\end{acks}

\bibliographystyle{ACM-Reference-Format}
\bibliography{bibl}

\appendix

\section{Motivation for Kronecker Product Approximations}
\label{app: kronecker products}

Here, we provide a complete description of the motivation for using a Kronecker product approximation for a matrix parameter arising from a fully-connected layer when training a multi-layer perceptron. Recall that the problem of neural network training is posed as
\begin{equation}
    \min_{w \in \mathbb{R}^d} \left\{ f(w) = \mathbb{E}_{(x, y) \sim \mathcal{D}}[\ell(m(w; x); y)] \right\}
\end{equation}
where the multi-layer perceptron model is defined as
\begin{equation}
    m(w; x) = W^{(n)} \sigma( W^{(n - 1)} \sigma( ... \sigma( W^{(1)} x) ... ) ),
\end{equation}
with $w = (\vect(W^{(1)})^T, ..., \vect(W^{(n)})^T)^T$. In order to examine its structure, we would like to derive full-matrix AdaGrad for a single parameter $W^{(i)}$.

For a single datapoint $(x, y) \sim \mathcal{D}$, we can isolate the problem for parameter $W^{(i)}$ by defining the function
\begin{equation}
    f^{(i)}(W) = \phi^{(i)}(W a^{(i - 1)}).
\end{equation}
Here, the activation $a^{(i - 1)}$ before layer $i$ and the function $\phi^{(i)} : \mathbb{R}^{d_i} \rightarrow \mathbb{R}$ are defined as: 
\begin{align}
    a^{(i - 1)} & = \sigma( W^{(i - 1)} ... \sigma(W^{(2)} \sigma( W^{(1)} x) ) ... ) \\
    \phi^{(i)}(z) & = \ell(W^{(n)} \sigma( W^{(n - 1)} \sigma( ... \sigma(z) ... )), y).
\end{align}
Note that $a^{(i - 1)}$ has an implicit dependence on $x$ and $\phi^{(i)}$ has an implicit dependence on $y$. This structure also holds for simpler machine learning models, such as multi-class logistic regression. The gradient in both matrix and vector form for a single sample may therefore be derived as:
\begin{align}
    \nabla f^{(i)}(W^{(i)}) & = \nabla \phi(z)|_{z = W^{(i)} a^{(i - 1)}} (a^{(i - 1)})^T \\ 
    \vect(\nabla f^{(i)}(W^{(i)})) & = \nabla \phi(z)|_{z = W^{(i)} a^{(i - 1)}} \otimes a^{(i - 1)}.
\end{align}

Let the subscript $s$ denote the gradient, function, or activation at iteration $s$. We can therefore expand the definition of full-matrix AdaGrad as
\begin{align*}
    A_t^{(i)} & = \sum_{s = 0}^t \vect(\nabla f_s^{(i)}(W_s^{(i)})) \vect(\nabla f_s^{(i)}(W_s^{(i)}))^T \\
    & = \sum_{s = 0}^t (\nabla \phi_s^{(i)}(z)|_{z = W^{(i)}_s a^{(i - 1)}_s} \otimes a^{(i - 1)}_s) (\nabla \phi_s^{(i)}(z)|_{z = W^{(i)}_s a^{(i - 1)}_s} \otimes a^{(i - 1)}_s)^T \\
    & = \sum_{s = 0}^t (\nabla \phi^{(i)}_s(z)|_{z = W^{(i)}_s a^{(i - 1)}_s} (\nabla \phi^{(i)}_s(z)|_{z = W^{(i)}_s a^{(i - 1)}_s})^T) \otimes (a^{(i - 1)}_s (a^{(i - 1)}_s)^T).
\end{align*}
where $(x_s, y_s) \sim \mathcal{D}$ is sampled at iteration $s$. So $A_t$ is in fact a sum of Kronecker products.

\section{Per-Parameter Relationship with Row-Wise AdaGrad and AdaFactor}
\label{app: relationships}

Row-wise AdaGrad and AdaFactor are two optimizers with sublinear memory cost designed for optimizing embedding tables and large language models, respectively \cite{mudigere2022software, shazeer2018adafactor, gupta2014training}. We will show that two separate versions of diagonal Shampoo are, in fact, equivalent to both AdaFactor and row-wise AdaGrad when applied to a \textit{single} matrix parameter $W \in \mathbb{R}^{m \times n}$. (These equivalences will \textit{not hold}, however, for the general multi-parameter case.)

We will rely on the following generalization of the Shampoo update equation \eqref{eq:shampoo update}
\begin{equation} \label{eq:generalized update}
    W_{t + 1} = W_t - \bar{\alpha}_t L_t^{-1/2p} G_t R_t^{-1/2q}
\end{equation}
for $\frac{1}{p} + \frac{1}{q} = 1$ with $p, q > 0$; see \cite{anil2021scalable}.

\subsection{Equivalence to Row-Wise AdaGrad}

Row-wise AdaGrad is an approximation to diagonal AdaGrad designed specifically for training large-scale embedding tables that arise in ranking and recommendation models \cite{mudigere2022software, gupta2014training}. Categorical features are often encoded using a one-hot or multi-hot encoding $x \in \mathbb{R}^m$ where each component of the encoding takes values in $\{0, 1\}$:
\begin{equation}
    x_i = 
    \begin{cases}
        1 & \mbox{ if category $i$ is active,} \\
        0 & \mbox{ otherwise.}
    \end{cases}
\end{equation}
Note that $m$ corresponds to the total number of categories. In the case of ranking and recommendation models, the categories are normally hashed via a modulo function, and therefore $m$ will correspond to the number of hash buckets \cite{chen2015compressing, shi2020compositional, naumov2019deep}. 

Embedding tables map individual or small subsets of categories to dense \textit{embedding vector} representations in some lower-dimensional space. In the simplest case, an embedding table $E \in \mathbb{R}^{m \times n}$ maps a categorical feature (represented by a one-hot or multi-hot encoding) $x \in \mathbb{R}^m$ to an embedding vector $e \in \mathbb{R}^n$ of dimension $n$ by computing
\begin{equation}
    e = E^T x.
\end{equation}
In the case of one-hot encodings, this corresponds to a single-row lookup. In the multi-hot encoding setting, this corresponds to performing \textit{sum pooling} over all active rows. Our corresponding optimization problem is therefore:
\begin{equation}
    \min_{E \in \mathbb{R}^{m \times n}} \mathbb{E}_{x \sim \mathcal{X}}[f(E; x)] = \mathbb{E}_{x \sim \mathcal{X}}[\phi(E^T x)]
\end{equation}
where $\phi: \mathbb{R}^n \rightarrow \mathbb{R}$. Note that the gradient for a single sample can be derived as:
\begin{equation}
    \nabla f(E; x) = x (\nabla \phi(y)|_{y = E^T x})^T.
\end{equation}

Row-wise AdaGrad approximates diagonal AdaGrad by only storing a single optimizer state for each row by averaging over the squared gradients of each row of the embedding table. In particular, let the row-wise AdaGrad optimizer state be denoted by $v \in \mathbb{R}^m$ with $v_{-1} = \epsilon 1_m$ for $\epsilon > 0$. Then row-wise AdaGrad performs the update:
\begin{align}
    v_t & = v_{t - 1} + \frac{\sum_{j = 1}^n [G_t^{\odot 2}]_{:, j}}{n} \\
    E_{t + 1} & = E_t - \alpha_t G_t / \sqrt{v_t 1_{m}^T}.
\end{align}
Note that this can be re-written as:
\begin{align}
    V_t & = V_{t - 1} + \frac{\matdiag(G_t G_t^T)}{n} \\
    E_{t + 1} & = E_t - \alpha_t V_t^{-1/2} G_t
\end{align}
with $V_{-1} = \epsilon I_m$. With some slight re-writing, we obtain:
\begin{align}
    \tilde{L}_t & = \tilde{L}_{t - 1} + \matdiag(G_t G_t^T) \\
    E_{t + 1} & = E_t - \bar{\alpha}_t \tilde{L}_t^{-1/2} G_t,
\end{align}
where $L_{-1} = \bar{\epsilon} I_m$ with $\bar{\epsilon} = \epsilon / n$ and $\bar{\alpha}_t = \alpha_t / n$. Note that this is precisely diagonal Shampoo (with a modified learning rate and epsilon) with update formula \eqref{eq:generalized update} with $p = 1$ and $q = \infty$! 

\subsection{Relationship with AdaFactor}

AdaFactor is a sublinear memory adaptive gradient method designed for training large language models \cite{shazeer2018adafactor}. It approximates diagonal AdaGrad for matrices by using a rank-one approximation. By observing that all entries of AdaGrad are non-negative, one can minimize the I-divergence between diagonal AdaGrad and its rank-one approximation and obtain a closed form solution expressed in terms of row sums and column sums, which are linear functions of the second moment. 

In particular, suppose we are interested in minimizing the matrix function
\begin{equation*}
    \min_{W \in \mathbb{R}^{m \times n}} \mathbb{E}_{x \sim \mathcal{X}}[f(W; x)] = \mathbb{E}_{x \sim \mathcal{X}}[\phi(W x)],
\end{equation*}
where $W \in \mathbb{R}^{m \times n}$, $x \in \mathbb{R}^n$ is sampled from some underlying probability distribution $\mathcal{X}$, and $\phi: \mathbb{R}^m \rightarrow \mathbb{R}$. Then AdaFactor (ignoring bias correction) will store two vectors $r \in \mathbb{R}^m$ and $c \in \mathbb{R}^n$ for the row and column sums, respectively, and update the parameters as:
\begin{align}
    r_t & = \beta_2 r_{t - 1} + (1 - \beta_2) [G_t^{\odot 2}] 1_n \\
    c_t & = \beta_2 c_{t - 1} + (1 - \beta_2) [G_t^{\odot 2}]^T 1_m \\
    \hat{A}_t & = r_t c_t^T / 1_m^T r_t \\
    W_{t + 1} & = W_t - \alpha_t G_t / (\sqrt{\hat{A}_t} + \epsilon 1_m 1_n^T).
\end{align}
Note that this can be re-written as
\begin{align}
    \tilde{L}_t & = \beta_2 \tilde{L}_{t - 1} + (1 - \beta_2) \matdiag(G_t G_t^T) \\
    \tilde{R}_t & = \beta_2 \tilde{R}_{t - 1} + (1 - \beta_2) \matdiag(G_t^T G_t) \\
    W_{t + 1} & = W_t - \bar{\alpha}_t \tilde{L}_t^{-1/2} G_t \tilde{R}_t^{-1/2} 
\end{align}
with $\bar{\alpha}_t = \sqrt{1_m^T L_t 1_m} \alpha_t$. This shares the same functional form as diagonal Shampoo (with exponential moving averaging of the factor matrices), except for a key difference with the exponent, where AdaFactor uses $-1/2$ as opposed to $-1/4$ with Shampoo. 

\section{Proofs}
\label{app: proofs}

For completeness, we present the algorithms and theorems with proof below. We will focus on the iterate-dependent case for generality. Recall that the momentum method \eqref{eq:momentum1}-\eqref{eq:momentum2} can be vectorized and written as the iteration:
\begin{align}
    m_t & = \mu_t m_{t - 1} + p_t(w_t) \label{eq:full-momentum1}\\
    w_{t + 1} & = w_t - \alpha_t m_t \label{eq:full-momentum2}
\end{align}
for momentum hyperparameter $\mu_t > 0$ and learning rate $\alpha_t > 0$. Similarly, Nesterov acceleration can be written as
\begin{align}
    m_t & = \mu_{t - 1} m_{t - 1} + p_t(w_t) \label{eq:full-nesterov1} \\
    w_{t + 1} & = w_t - \alpha_t (\mu_t m_t + p_t(w_t)). \label{eq:full-nesterov2}
\end{align}
This generalizes the equivalence for SGD with momentum observed in \citet{defazio2020momentum}.

\begin{theorem}
The momentum method defined in \eqref{eq:full-momentum1}-\eqref{eq:full-momentum2} is equivalent to the heavy ball iteration
\begin{equation} \label{eq:heavyball}
    w_{t + 1} = w_t - \alpha_t p_t(w_t) + \delta_t (w_t - w_{t - 1})
\end{equation}
and the stochastic primal iterate averaging iteration
\begin{align}
    z_{t + 1} & = z_t - \eta_t p_t(w_t) \label{eq:full-iteravg1}\\
    w_{t + 1} & = (1 - c_t) w_t + c_t z_{t + 1} \label{eq:full-iteravg2}
\end{align}
with $c_t = \frac{\gamma_t}{\gamma_t + 1} \in [0, 1)$ for some $\gamma_t > 0$, $\delta_t = \frac{\alpha_{t - 1}}{\alpha_t} \mu_t = \frac{\gamma_{t - 1}}{\gamma_t + 1}$, and $\alpha_t = \frac{\eta_t}{\gamma_t + 1} = (1 - c_t) \eta_t$. In particular, if the hyperparameters are fixed, i.e., $\alpha_t \equiv \alpha$, $\mu_t \equiv \mu$, $c_t \equiv c$, $\eta_t \equiv \eta$, and $\delta_t \equiv \delta$, then $\alpha = (1 - c) \eta$ and $\delta = \mu = c$.
\end{theorem}

\begin{proof}
We will show that both methods may be written in the form of \eqref{eq:heavyball}. For the momentum equations \eqref{eq:full-momentum1}-\eqref{eq:full-momentum2}, note that if we plug in \eqref{eq:full-momentum1} into \eqref{eq:full-momentum2}, we obtain:
\begin{align*}
    w_{t + 1} & = w_t - \alpha_t p_t(w_t) - \mu_t \alpha_t m_{t - 1} \\
    & = w_t - \alpha_t p_t(w_t) + \delta_t (w_t - w_{t - 1}).
\end{align*}
where the last line follows since $w_t - w_{t - 1} = - \alpha_{t - 1} m_{t - 1}$ and if we choose $\delta_t = \frac{\alpha_{t - 1}}{\alpha_t} \mu_t$.

Now, if we plug in $c_t = \frac{\gamma_t}{\gamma_t + 1}$ into \eqref{eq:full-iteravg1}-\eqref{eq:full-iteravg2}, we obtain
\begin{align*}
    z_{t + 1} & = z_t - \eta_t p_t(w_t) \\
    w_{t + 1} & = \frac{\gamma_t}{\gamma_t + 1} w_t + \frac{1}{\gamma_t + 1} z_{t + 1}.
\end{align*}
Equivalently,
\begin{equation*}
    z_{t + 1} = (\gamma_t + 1) w_{t + 1} - \gamma_t w_t. 
\end{equation*}
Therefore,
\begin{align*}
    w_{t + 1} & = \frac{\gamma_t}{\gamma_t + 1} w_t + \frac{1}{\gamma_t + 1} (z_t - \eta_t p_t(w_t)) \\
    & = \frac{\gamma_t}{\gamma_t + 1} w_t + \frac{1}{\gamma_t + 1} ((\gamma_{t - 1} + 1) w_t - \gamma_{t - 1} w_{t - 1} - \eta_t p_t(w_t)) \\
    & = w_t - \frac{\eta_t}{\gamma_t + 1} p_t(w_t) + \frac{\gamma_{t - 1}}{\gamma_t + 1} (w_t - w_{t - 1}) \\
    & = w_t - \alpha_t p_t(w_t) + \delta_t (w_t - w_{t - 1})
\end{align*}
since $\alpha_t = \frac{\eta_t}{\gamma_t + 1}$ and $\delta_t = \frac{\gamma_{t - 1}}{\gamma_t + 1}$.
\end{proof}

Observe that this equivalence provides some guidelines for choosing the hyperparameters for momentum methods. In particular, if the momentum hyperparameter is fixed to $\mu_t \equiv \mu = 0.9$ (as is typically set in practice), then we should multiply the maximum initial learning rate $\eta_0 \leq \bar{\eta}_0$ by $0.1$ when using the momentum method to prevent divergence, i.e., $\alpha_0 \leq \bar{\alpha}_0 = 0.1 \bar{\eta}_0$. 

Typically, we apply learning rate warmup and/or different learning rate annealing techniques, for example by cutting the learning rate periodically or applying cosine annealing. In order to ensure that we are applying a consistent averaging hyperparameter $c_t \equiv c$, we need to adjust the momentum hyperparameter when the learning rate changes. For example, if $\alpha_t = \zeta \alpha_{t - 1}$, then we should set $\mu_t = \zeta \mu_{t - 1}$. Otherwise, fixing the momentum hyperparameter while decaying the learning rate corresponds to a larger averaging hyperparameter $c_t$.

We can show a similar theorem for the Nesterov acceleration case.

\begin{theorem}
The Nesterov accelerated method defined in \eqref{eq:full-nesterov1}-\eqref{eq:full-nesterov2} is equivalent to the modified heavy ball iteration
\begin{equation} \label{eq:heavyball-nesterov}
    w_{t + 1} = w_t - \alpha_t (p_t(w_t) + \mu_t (p_t(w_t) - p_{t - 1}(w_{t - 1}))) + \delta_t (w_t - w_{t - 1})
\end{equation}
and the modified stochastic primal iterate averaging iteration
\begin{align}
    z_{t + 1} & = z_t - \eta_t (p_t(w_t) + \mu_t (p_t(w_t) - p_{t - 1}(w_{t - 1})) \label{eq:nesterov-iteravg1}\\
    w_{t + 1} & = (1 - c_t) w_t + c_t z_{t + 1} \label{eq:nesterov-iteravg2}
\end{align}
with $c_t = \frac{\gamma_t}{\gamma_t + 1} \in [0, 1)$ for some $\gamma_t > 0$, $\delta_t = \frac{\alpha_{t - 1}}{\alpha_t} \mu_t = \frac{\gamma_{t - 1}}{\gamma_t + 1}$, and $\alpha_t = \frac{\eta_t}{\gamma_t + 1} = (1 - c_t) \eta_t$. In particular, if the hyperparameters are fixed, i.e., $\alpha_t \equiv \alpha$, $\mu_t \equiv \mu$, $c_t \equiv c$, $\eta_t \equiv \eta$, and $\delta_t \equiv \delta$, then $\alpha = (1 - c) \eta$ and $\delta = \mu = c$.
\end{theorem}

\begin{proof}
We will show that both methods may be written in the form \eqref{eq:heavyball-nesterov}. For the Nesterov accelerated equations \eqref{eq:full-nesterov1}-\eqref{eq:full-nesterov2}, note that if we plug in \eqref{eq:full-nesterov1} into \eqref{eq:full-nesterov2}, we obtain:
\begin{align*}
    w_{t + 1} & = w_t - \alpha_t (\mu_t m_t + p_t(w_t)) \\
    & = w_t - \alpha_t (\mu_t (\mu_{t - 1} m_{t - 1} + p_t(w_t)) + p_t(w_t)) \\
    & = w_t - \alpha_t \mu_t \mu_{t - 1} m_{t - 1} - \alpha_t (1 + \mu_t) p_t(w_t) \\
    & = w_t - \alpha_t p_t(w_t) - \mu_t \alpha_t (p_t(w_t) - p_{t - 1}(w_{t - 1})) - \mu_t \alpha_t (\mu_{t - 1} m_{t - 1} + p_{t - 1}(w_{t - 1})) \\
    & = w_t - \alpha_t (p_t(w_t) + \mu_t (p_t(w_t) - p_{t - 1}(w_{t - 1}))) + \delta_t (w_t - w_{t - 1})
\end{align*}
where the last line follows since $w_t - w_{t - 1} = - \alpha_{t - 1} (\mu_{t - 1} m_{t - 1} + p_{t - 1}(w_{t - 1}))$ and if we choose $\delta_t = \frac{\alpha_{t - 1}}{\alpha_t} \mu_t$.

Now, if we plug in $c_t = \frac{\gamma_t}{\gamma_t + 1}$ into \eqref{eq:nesterov-iteravg1}-\eqref{eq:nesterov-iteravg2}, we obtain
\begin{align*}
    z_{t + 1} & = z_t - \eta_t (\mu_t m_t + p_t(w_t)) \\
    w_{t + 1} & = \frac{\gamma_t}{\gamma_t + 1} w_t + \frac{1}{\gamma_t + 1} z_{t + 1}.
\end{align*}
Equivalently,
\begin{equation*}
    z_{t + 1} = (\gamma_t + 1) w_{t + 1} - \gamma_t w_t. 
\end{equation*}
Therefore,
\begin{align*}
    w_{t + 1} & = \frac{\gamma_t}{\gamma_t + 1} w_t + \frac{1}{\gamma_t + 1} (z_t - \eta_t (p_t(w_t) + \mu_t (p_t(w_t) - p_{t - 1}(w_{t - 1})))) \\
    & = \frac{\gamma_t}{\gamma_t + 1} w_t + \frac{1}{\gamma_t + 1} ((\gamma_{t - 1} + 1) w_t - \gamma_{t - 1} w_{t - 1} - \eta_t (p_t(w_t) + \mu_t (p_t(w_t) - p_{t - 1}(w_{t - 1})))) \\
    & = w_t - \frac{\eta_t}{\gamma_t + 1} (p_t(w_t) + \mu_t (p_t(w_t) - p_{t - 1}(w_{t - 1}))) + \frac{\gamma_{t - 1}}{\gamma_t + 1} (w_t - w_{t - 1}) \\
    & = w_t - \alpha_t (p_t(w_t) + \mu_t (p_t(w_t) - p_{t - 1}(w_{t - 1})) + \delta (w_t - w_{t - 1})
\end{align*}
since $\alpha_t = \frac{\eta_t}{\gamma_t + 1}$ and $\delta_t = \frac{\gamma_{t - 1}}{\gamma_t + 1}$.
\end{proof}

Observe that the primary difference between using momentum and Nesterov acceleration has to do with the inclusion of an additional correction term $\mu_t(p_t(w_t) - p_{t - 1}(w_{t - 1}))$ to the search direction based on the previous search direction in the iterative averaging scheme. Notice that if $p_t(w_t) \approx p_{t - 1}(w_{t - 1})$, then no correction is performed. The practical effectiveness of such a correction term is still not well understood and left for future research. 

\section{Experimental Details}

\subsection{Method Hyperparameters}
\label{sec:appx:hyperparameters}

Tables \ref{tab:shampoo_hyperparameters} and \ref{tab:nesterov_hyperparameters} contain the common hyperparameters used for our experiments with Shampoo and SGD with Nesterov, respectively. These hyperparameters are used in all experiments, unless explicitly mentioned otherwise.

\begin{table}[htbp]
  \footnotesize
  \centering
  \renewcommand{\arraystretch}{0.85} 
  \begin{tabular}{@{}p{0.45\linewidth} p{0.25\linewidth}@{}}
    \toprule
    \textbf{Parameter} & \textbf{Value} \\
    \midrule
    
    \textbf{Basic Shampoo parameters} & \\
    \quad \texttt{lr} & \textit{See LR scheduler section} \\
    \quad \texttt{betas} & [0., 0.999] \\
    \quad \texttt{momentum} & 0.9 \\
    \quad \texttt{use\_nesterov} & \texttt{True} \\
    \quad \texttt{use\_bias\_correction} & \texttt{True} \\
    \quad \texttt{weight\_decay} & 1e-4 \\
    \quad \texttt{use\_decoupled\_weight\_decay} & \texttt{True} \\
    \quad \texttt{use\_weight\_decay\_on\_bias\_and\_norm\_params} & \texttt{True} \\
    \midrule

    \textbf{Accelerators} & \\
    \quad Number of GPUs & 8 Tesla V100 \\
    \quad \texttt{batch\_size\_per\_gpu} & 128 \\
    \quad \texttt{num\_trainers\_per\_group} & -1 \\
    \midrule
    
    \textbf{Preconditioning} & \\
    \quad \texttt{max\_preconditioner\_dim} & 2048 \\
    \quad \texttt{precondition\_frequency} & 50 \\
    \quad \texttt{start\_preconditioning\_step} & 0 \\
    \quad \texttt{preconditioner\_dtype} & \texttt{torch.float} \\
    \quad \texttt{large\_dim\_method} & Blocking \\
    \quad \texttt{use\_merge\_dims} & \texttt{True} \\
    \midrule

    \textbf{Preconditioner exponent heuristics} & \\
    \quad \texttt{exponent\_override} & $2 \omega$  \graytext{(Default)}\\
    \quad \texttt{exponent\_multiplier} & 1.0  \graytext{(Default)}\\
    \midrule
    
    \textbf{Grafting} & \\
    \quad \texttt{grafting\_type} & SGD \\
    \quad \texttt{grafting\_epsilon} & 1e-8 \\
    \quad \texttt{grafting\_beta2} & 0.999 \\
    \midrule
    
    \textbf{LR scheduler} & \\
    \quad Name & Cosine decay with warmup \\
    \quad \texttt{initial\_lr} & 0.1 \\
    \quad \texttt{warmup\_epochs} & 5 \\
    \midrule
    
    \textbf{Root inverse} & \\
    \quad \texttt{root\_inv\_method} & Eigendecomposition \\
    \quad \texttt{epsilon} & 1e-12 \\
    \quad \texttt{use\_protected\_eigh} & \texttt{True} \\
    \bottomrule
  \end{tabular}
  \caption{Generic Shampoo settings for ImageNet task}
  \label{tab:shampoo_hyperparameters}
\end{table}

\begin{table}[htbp]
  \footnotesize
  \centering
  \renewcommand{\arraystretch}{0.85} 
  \begin{tabular}{@{}p{0.45\linewidth} p{0.25\linewidth}@{}}
    \toprule
    \textbf{Parameter} & \textbf{Value} \\
    \midrule
    
    \textbf{Basic Nesterov parameters} & \\
    \quad \texttt{lr} & \textit{See LR scheduler section} \\
    \quad \texttt{betas} & [0., 0.999] \\
    \quad \texttt{momentum} & 0.9 \\
    \quad \texttt{nesterov} & \texttt{True} \\
    \quad \texttt{dampening} & 0 \\
    \quad \texttt{use\_weight\_decay\_on\_bias\_and\_norm\_params} & \texttt{True} \\
    \midrule

    \textbf{Accelerators} & \\
    \quad Number of GPUs & 8 Tesla V100 \\
    \quad \texttt{batch\_size\_per\_gpu} & 128 \\
    \midrule
    
    \textbf{LR scheduler} & \\
    \quad Name & Cosine decay with warmup \\
    \quad \texttt{initial\_lr} & 0.1 \\
    \quad \texttt{warmup\_epochs} & 5 \\
    
    \bottomrule
  \end{tabular}
  \caption{Generic Nesterov settings for ImageNet task}
  \label{tab:nesterov_hyperparameters}
\end{table}

\subsection{Performance-Related Ablations}
\label{sec:appx:performance_ablations}

Two of the key hyperparameters controlling the (computational) performance of Shampoo are: (i) the frequency with which the preconditioners are updated, and (ii) the maximum allowed preconditioner size. In this section, we perform ablations of these hyperparameters to demonstrate robustness to their choices in terms of accuracy, allowing us to select settings for these two hyperparameters focused on reducing wall-clock time. 

The tables presented below contain information about the accuracy of Shampoo using different hyperparameter configurations compared to Nesterov. We compare both under a fixed epoch budget of 90 epochs as well as a time budget based on the total wall-clock time required by Nesterov to complete 90 epochs.

\subsubsection{Preconditioner (Update) Frequency}

We begin by studying the effect of the frequency of preconditioner updates. Recall that \textit{all} stochastic gradients are preconditioned. However some of the used preconditioners may be stale. The degree of staleness is controlled by the \texttt{precondition\_frequency} hyperparameter. In this ablation, we fix all configurations of Shampoo as in Table \ref{tab:shampoo_hyperparameters}, except for the \texttt{max\_preconditioner\_dim} which is set to 1024. Results are presented in Table \ref{tab:precondition_frequency_ablation}.

\begin{table}[h]
    \centering
    \begin{tabular}{ccccc}
        \hline
        \textbf{Preconditioner} & \textbf{Val. accuracy} & \textbf{Val. accuracy} & \textbf{Steps} & \textbf{Time overhead} \\ 
        \textbf{update frequency} & \textbf{at 90 epochs} & \textbf{at eq. time} & \textbf{per second} & \textbf{wrt Nesterov} \\ 
        \hline
        \textit{Nesterov}   &   76.937   &   76.937   &   2.373   &   --- \\
        1          &   77.476   &   62.031   &   1.12   &   53.1\%   \\
        20         &   77.495   &   75.444   &   2.07   &   12.80\% \\
        50         &   77.454   &   77.135   &   2.25   &   5.15\% \\
        100        &   77.442   &   77.473   &   2.32   &   1.98\% \\
        \hline
    \end{tabular}
    \caption{Effect of the preconditioner update frequency in Shampoo.}
    \label{tab:precondition_frequency_ablation}
\end{table}

Shampoo exhibits marginal differences in validation accuracy after 90 epochs, with significant computational performance changes depending on the choice of update frequency. As expected, carrying out very frequent preconditioner updates can induce a prohibitively large overhead in the execution of Shampoo. Fortunately, results demonstrate that it is possible use stale preconditioners successfully, corroborating \cite{anil2021scalable}. It is possible to recover all of the accuracy gains obtained as when preconditioning at every step  by using preconditioners updated every 50 or 100 steps, with only a 2-5\% computational overhead. 

Based on these results, we decided to set the \texttt{precondition\_frequency} hyperparameter to 50 for all other experiments. 

\subsubsection{Maximum Preconditioner Dimension}

This section studies the effect of the maximum allowed preconditioner size. This hyperparameter trades off between a closer approximation to block full-matrix Adagrad (which corresponds to a preconditioner containing all variables within each parameter) and the cost of linear algebra operations such as the inverse root computation (which may have a cubic dependency on the preconditioner size). While theoretically decreasing the block size should yield a smaller FLOP count, this is not necessarily the case due to CPU-GPU latency and additional kernel launches. As per the previous section, for this ablation we fix the \texttt{precondition\_frequency} to 50 and all other hyperparameters as in Table \ref{tab:shampoo_hyperparameters}.

Table \ref{tab:max_preconditioner_dim_ablation} shows that the effect of the maximum preconditioner size is much less significant than that of the preconditioner update frequency. As expected, there is a monotonic cost increase when allowing larger preconditioner sizes. However, the small increases in overhead (~1\% per factor of 2 in preconditioner size) demonstrate that the preconditioner inversion is not the main bottleneck in the execution of a weight update for Shampoo.

Working under the assumption that a closer approximation to the full-matrix preconditioner is desirable, and given that the additional time overhead between values 2048 vs 1024 was about 1.5\%, we decided to set \texttt{max\_preconditioner\_dim} of 2048 for our experiments.

\begin{table}[h]
    \centering
    \begin{tabular}{ccccc}
        \hline
        \textbf{Max preconditioner} & \textbf{Val. accuracy} & \textbf{Val. accuracy} & \textbf{Steps} & \textbf{Time overhead} \\ 
        \textbf{dimension} & \textbf{at 90 epochs} & \textbf{at eq. time} & \textbf{per second} & \textbf{wrt Nesterov} \\ 
        \hline
        \textit{Nesterov}   &   76.937   &   76.937   &   2.373   &   --- \\
        8192        &   77.419   &   76.547   &   2.173   &   8.40\% \\
        4096        &   77.266   &   76.582   &   2.193   &   7.54\% \\
        2048        &   77.507   &   76.817   &   2.213   &   6.70\% \\
        1024        &   77.454   &   77.081   &   2.250   &   5.15\%   \\
        \hline
    \end{tabular}
    \caption{Effect of the maximum preconditioner dimension in Shampoo.}
    \label{tab:max_preconditioner_dim_ablation}
\end{table}

\section{Implementation Structure}

The implementation currently has the following class structure:
\begin{itemize}
    \item \texttt{DistributedShampoo(torch.optim.Optimizer)} [\texttt{distributed\_shampoo.py}]: Main optimizer class. Depending on the selected large-dimensional method, uses different \texttt{Preconditioner} classes, including: \\
    
    \begin{itemize}
        \item \texttt{BlockShampooPreconditioner(DistributedPreconditioner)} [\texttt{shampoo\_utils.py}]: Block Shampoo preconditioner class that blocks the parameters and applies Shampoo to each block. Contains a \texttt{ShampooPreconditioner} object for each layer's block.\\
        \item \texttt{ShampooPreconditioner(DistributedPreconditioner)} [\texttt{shampoo\_utils.py}]: Shampoo preconditioner class that provides functionality for applying Shampoo to a particular parameter. Constructs multiple Kronecker factors corresponding to the number of dimensions. Implemented using \texttt{ShampooKroneckerFactor} and \texttt{Grafting} classes.\\
        
        \begin{itemize}
            \item \texttt{ShampooKroneckerFactor(OptimizerModule)} [\texttt{shampoo\_utils.py}]: Data class containing the factor matrix, the root inverse factor matrix, and additional metadata for distributed computation.\\
        \end{itemize}
        
        \begin{itemize}
            \item \texttt{SGDGrafting(Grafting)} [\texttt{shampoo\_utils.py}]: SGD grafting class. \\
            \item \texttt{AdaGradGrafting(Grafting)} [\texttt{shampoo\_utils.py}]: AdaGrad grafting class. Contains \texttt{AdaGradPreconditioner} class.\\
            \item \texttt{RMSPropGrafting(AdaGradGrafting)} [\texttt{shampoo\_utils.py}]: RMSProp grafting class built on \texttt{AdaGradGrafting}.\\
            \item \texttt{AdamGrafting(AdaGradGrafting)} [\texttt{shampoo\_utils.py}]: Adam grafting class built on \texttt{AdaGradGrafting}.\\
            \item \texttt{AdaGradNormalizedGrafting(AdaGradGrafting)} [\texttt{shampoo\_utils.py}]: Normalized AdaGrad grafting class built on \texttt{AdaGradGrafting}. Normalizes the gradient before using it to update the AdaGrad preconditioner.\\
            \item \texttt{RMSPropNormalizedGrafting(AdaGradGrafting)} [\texttt{shampoo\_utils.py}]: Normalized RMSProp grafting class built on \texttt{AdaGradGrafting}. Normalizes the gradient before using it to update the RMSProp preconditioner.\\
            \item \texttt{AdamNormalizedGrafting(AdaGradGrafting)} [\texttt{shampoo\_utils.py}]: Normalized Adam grafting class built on \texttt{AdaGradGrafting}. Normalizes the gradient before using it to update the Adam preconditioner.\\
        \end{itemize}
        
        \item \texttt{AdaGradPreconditioner(DistributedPreconditioner)} [\texttt{shampoo\_utils.py}]: AdaGrad preconditioner class that implements AdaGrad and RMSProp. Uses standard summation if \texttt{beta2 = 1}, otherwise uses exponential moving averaging.
    \end{itemize}
\end{itemize}

We build on top of the following abstract data structures:
\begin{itemize}
    \item \texttt{OptimizerModule} [\texttt{optimizer\_modules.py}]: Generic optimizer module that enables recursive generation of a nested state dictionary for optimizers. Supported by \texttt{TorchRec}'s \texttt{KeyedOptimizer}'s \texttt{state\_dict} and \texttt{load\_state\_dict} functionality, which is necessary for checkpointing \cite{ivchenko2022torchrec}.
    \item \texttt{Preconditioner(OptimizerModule)} [\texttt{shampoo\_utils.py}]: Generic preconditioner class containing functions for updating the preconditioner and preconditioning a vector.
    \item \texttt{DistributedPreconditioner(Preconditioner)} [\texttt{shampoo\_utils.py}]:  Contains all distributed buffer functionality necessary for distributing the computation.
\end{itemize}

\end{document}